\documentclass{article}

\PassOptionsToPackage{round}{natbib}

\usepackage[preprint]{neurips_2026}

\usepackage[utf8]{inputenc}
\usepackage[T1]{fontenc}
\usepackage{lmodern}
\usepackage[hidelinks]{hyperref}
\usepackage{url}
\usepackage{booktabs}
\usepackage{amsfonts}
\usepackage{amsmath,amssymb,amsthm,mathtools,bm}
\usepackage{nicefrac}
\usepackage{microtype}
\usepackage{xcolor}
\usepackage{graphicx}
\usepackage{multirow}
\usepackage{caption}
\usepackage{subcaption}
\usepackage{enumitem}
\usepackage{algorithm}
\usepackage{algpseudocode}
\usepackage{tikz}
\usetikzlibrary{positioning,arrows.meta,fit,calc}

\newtheorem{proposition}{Proposition}
\theoremstyle{remark}
\newtheorem{remark}{Remark}

\title{Spatial Adapter: Structured Spatial Decomposition and Closed-Form Covariance for Frozen Predictors}

\author{%
  Wen-Ting Wang \\
  Institute of Statistics, National Chung Hsing University \\
  Department of Applied Mathematics, National Dong Hwa University \\
  \texttt{egpivo@gmail.com}
  \And
  Wei-Ying Wu\thanks{Partially supported by
    NSTC 113-2118-M-259-002-MY2 and NDHU Funding 114T2560-03.} \\
  Department of Applied Mathematics \\
  National Dong Hwa University \\
  \texttt{wuweiying1011@gms.ndhu.edu.tw}
  \AND
  Hao-Yun Huang\thanks{Corresponding author.  Partially supported by
    NSTC 113-2118-M-259-001-MY2 and NDHU Funding 114T2560-03.} \\
  Department of Applied Mathematics \\
  National Dong Hwa University \\
  \texttt{hhuscout@gms.ndhu.edu.tw}
  \And
  Xuan-Chun Wang \\
  Department of Applied Mathematics \\
  National Dong Hwa University \\
  \texttt{wangxc1117@gmail.com}
}

\begin{document}

\maketitle

\begin{abstract}
We present the \textbf{Spatial Adapter}, a parameter-efficient
post-hoc layer that equips any frozen first-stage predictor with a
\emph{structured spatial representation} of its residual field and
an induced closed-form spatial covariance.  The adapter operates as
a cascade second stage on residuals, jointly learning a spatially
regularized orthonormal basis and per-sample scores via a tractable
mini-batch ADMM procedure, without modifying any first-stage
parameter.  Because the first-stage parameters are frozen, the adapter does
not retrain the backbone; its role is to supply a compressed
distributional summary of the residual field.  Smoothness, sparsity, and orthogonality together turn a
generic low-rank factorization into an identifiable spatial
representation whose induced residual covariance admits a
closed-form low-rank-plus-noise estimator; the \emph{effective
rank} is determined data-adaptively by spectral thresholding, while
the nominal rank $K$ is an optimization-side upper bound only.  This
covariance enables kriging-style spatial prediction at unobserved
locations, with plug-in uncertainty quantification as a secondary
downstream use.  Across synthetic data, Weather2K (spatial-holdout prediction),
and GWHD patch grids (basis-transferability diagnostic), the
adapter recovers residual spatial structure when paired with
frozen first stages from linear models to deep spatiotemporal
and vision backbones;
the added representation uses fewer than $K(N{+}T)$ parameters
alongside a compact residual-trend network.
\end{abstract}

\section{Introduction}

Large pretrained neural backbones~\citep{he2016deep,dosovitskiy2021vit}
now dominate vision, language, and scientific modelling, and a
thriving parameter-efficient adaptation
literature~\citep{hu2022lora,houlsby2019adapter} lets practitioners
steer them to new tasks at a fraction of the cost of retraining.  Yet
for any problem where the response lives on a discrete set of spatial
locations $\{\mathbf s_i\}_{i=1}^{N}$---weather stations scattered
across a region, pixels or patches on a regular image grid, cells in
a simulation mesh, or sensor arrays in a building---the predictor's
point-prediction loss places no constraint on the spatial structure
of its residuals, leaving systematic neighbour-to-neighbour
correlation unmodelled in the output.  Spatial statistics has exactly
the tools for this gap: hierarchical Gaussian
processes~\citep{gelfand2003spatial}, fixed-rank
kriging~\citep{cressie2008fixed}, regularized spatial
PCA~\citep{wang2017regularized}, empirical orthogonal function
(EOF) decompositions long standard in climate and geophysical
analysis, and resolution-adaptive
extensions~\citep{tzeng2018resolution}, all rooted in low-rank or
sparse decompositions of the spatial covariance---and all
delivering \emph{closed-form covariance-based uncertainty
quantification} as a direct by-product.  Modern scientific
deep-learning pipelines, however, tend to replace these tools
with opaque representations that improve point accuracy but
shed the covariance structure, the explicit low-rank factors,
and the closed-form uncertainty quantification that the classical
tools provided---a tension between predictive performance and
statistical interpretability recently highlighted from the
statistics side by~\citet{wikle2023statistical}.  Compounding
the gap, these classical methods also scale poorly in $N$ and
do not naturally interoperate with deep backbones.
Classical geostatistics decomposes the response as
$\mathbf Y(t,\mathbf s)=\mu(t,\mathbf s)+w(t,\mathbf s)+\bm\epsilon(t,\mathbf s)$,
with $\mu$ a mean trend, $w$ a zero-mean spatial process, and
$\bm\epsilon$ white noise; our goal is to recover this decomposition in modern predictive
pipelines by taking $\mu$ to be a pretrained first-stage predictor
and replacing
$w$ with a learnable low-rank spatial representation, without
retraining the first stage.

\paragraph{Spatial Adapter.}
We introduce the \emph{Spatial Adapter}, a post-hoc layer that
extracts structured spatial signal from the residual field of any
frozen first-stage predictor into a low-rank representation
$\bm\Phi(\mathbf s)^{\!\top}\bm\alpha_{j}$.  Because the first-stage
parameters are frozen, the adapter does not retrain the backbone;
rather, it supplies a compressed distributional summary of the
residual field from which an explicit spatial covariance follows
directly, with interval estimation as one downstream use of the
same covariance machinery.
Stage~1 fits a first-stage predictor $f_{\bm\theta}$ by any
standard recipe.  Stage~2 freezes $\bm\theta$ and jointly fits a
compact residual-trend network $\mu_{\bm\psi}$, a
spatially-smooth orthonormal basis $\bm\Phi$, and per-sample
scores $\bm\alpha_{j}$ (indexed by time step or by image,
depending on the task) by minimizing the task loss plus spatial
regularizers, via a three-block mini-batch
ADMM~\citep{boyd2011distributed}.  The spatial index $\mathbf s$
is general: the same machinery applies to geographic station
coordinates and to within-image patch grids alike.  The
Gaussian/identity case admits a closed-form Z-update and an
eigendecomposition-based basis update inside the ADMM skeleton
(the trend step remains a gradient step in both cases), while
the Bernoulli/sigmoid case uses the same skeleton with
surrogate-based approximate updates (including a variance-only
simplification) that we validate empirically.
Plugging the learned $\widehat{\bm\Phi}$ into the fixed-rank
decomposition of~\citet{wang2017regularized} immediately yields a
closed-form spatial covariance estimator at no extra optimization
cost.  This covariance enables two operational modes: spatial
\emph{reconstruction} at observed locations, where the adapter
refines the first-stage prediction by projecting residuals onto the
learned basis; and spatial \emph{prediction} at unobserved
locations, where kriging-style conditioning on the estimated
covariance yields a point predictor together with a task-agnostic
closed-form covariance and plug-in variance machinery; interval
construction on the probability scale is illustrated diagnostically
in the appendix.

\paragraph{Contributions.}\label{sec:intro_contributions}
\begin{enumerate}[leftmargin=1.5em,itemsep=1pt]
  \item \textbf{Residual spatial representation for frozen
        predictors.}  The Spatial Adapter operates on the residual
        field of a frozen first-stage predictor rather than on its
        parameters, learning an explicit low-rank representation
        with a spatially regularised orthonormal basis $\bm\Phi$
        and per-sample scores $\bm\alpha_{j}$.
  \item \textbf{Structured covariance induced by the representation.}
        $\widehat{\bm\Phi}$ directly induces a closed-form
        low-rank-plus-noise covariance
        estimator~\citep{wang2017regularized}; the \emph{effective}
        rank is selected data-adaptively by spectral thresholding
        rather than fixed at the optimization-side rank $K$.
  \item \textbf{Mini-batch ADMM with closed-form Gaussian updates.}
        A three-block mini-batch ADMM fits the adapter.  The
        Gaussian/identity consensus update is closed form
        (Proposition~\ref{prop:z-update}) and the basis step is an
        eigendecomposition; the Bernoulli case reuses the skeleton
        with surrogate-based approximate updates
        (Remark~\ref{rem:bce-basis-step}).
  \item \textbf{Empirical validation.}  Evaluated on synthetic data,
        Weather2K (spatial-holdout prediction), and GWHD patch
        grids (basis-transferability diagnostic), paired with first
        stages from OLS to deep spatiotemporal and vision backbones.
        Kriging-style prediction and plug-in logistic-normal
        intervals are reported as downstream uses of the learned
        covariance, not separate primary contributions.
\end{enumerate}

\paragraph{Positioning.}
We inherit the post-hoc, frozen-backbone philosophy of the
adapter/LoRA tradition~\citep{houlsby2019adapter,hu2022lora,han2024peftsurvey,liu2024dora,yang2024bayesianlora}
but operate on the residual field rather than on model
weights---placing the structured decomposition on what the
predictor \emph{failed to explain}, not on how it computes.
Unlike parameter-space PEFT variants that modify or augment model
weights, the Spatial Adapter does not adapt weights; it learns a
structured low-rank representation on the \emph{residual field}
of a frozen predictor.  The crucial
ingredient is not low rank alone but a spatially regularized
functional structure (smoothness, local support), which makes
closed-form covariance estimation
$\bm\Phi\,\Lambda\,\bm\Phi^{\!\top}$ a direct by-product.  The
closest statistical precursors are structured basis-function
models in spatial statistics---fixed-rank
kriging~\citep{cressie2008fixed}, regularized spatial
PCA~\citep{wang2017regularized}, and their
extensions~\citep{tzeng2018resolution}; the Spatial Adapter
recovers SpatPCA as a limiting case (Section~\ref{sec:methodology}).
In short, the adapter is primarily a representation tool;
closed-form covariance is a natural consequence of the
identifiable, spatially regularized structure, and interval
estimation is one downstream statistical use of that covariance.

\section{Methodology}\label{sec:methodology}

We cast the Spatial Adapter as a generalized linear model
(Section~\ref{subsec:model}) whose latent field decomposes additively
into a frozen first-stage trend, a compact residual correction,
and an explicit low-rank spatial representation, fitted in two stages
by a mini-batch ADMM
(Sections~\ref{subsec:twostage}--\ref{subsec:admm}).  As a byproduct of Stage~2, the learned basis $\widehat{\bm\Phi}$
yields a closed-form spatial covariance estimator
$\widehat\Sigma_{\mathbf r}$ (Section~\ref{subsec:cov-est}) that
enables two canonical uses of a spatial covariance---kriging-style
prediction at unobserved locations and a link-scale plug-in
variance/interval machinery.  The kriging use is derived in
Section~\ref{subsec:cov-est} and instantiated in the Weather2K
spatial-holdout experiment (Section~\ref{subsec:weather2k});
the link-scale logistic-normal pushforward to the binary case is
illustrated diagnostically in Appendix~\ref{app:bce-ablation}.  Figure~\ref{fig:method-overview}
gives the architectural overview.

\begin{figure}[t]
\centering
\resizebox{0.92\linewidth}{!}{%
\begin{tikzpicture}[
  font=\small,
  >={Stealth[length=2mm]},
  thick,
  box/.style={draw, rounded corners=3pt, align=center, inner sep=3pt},
  frozen/.style={box, fill=gray!12, draw=gray!55!black, text width=2.0cm},
  adapter/.style={box, fill=blue!8,  draw=blue!55!black, text width=2.2cm},
  repr/.style={box, fill=blue!3,  draw=blue!35!black, text width=2.2cm},
  refine/.style={box, fill=green!15, draw=green!45!black, text width=2.0cm},
  stats/.style={box, fill=orange!12, draw=orange!65!black, text width=2.0cm}
]
\node[frozen] (backbone) at (0,0) {
  \textbf{Frozen first-stage}\\
  \textbf{predictor}\\[2pt]
  \scriptsize residual spatial\\ structure
};
\node[adapter] (adapt) at (3.2,0) {
  \textbf{Spatial Adapter}\\[2pt]
  \scriptsize residual trend +\\
  structured low-rank\\
  correction
};
\node[repr] (repr) at (6.4,0) {
  \textbf{Explicit spatial}\\
  \textbf{representation}\\[2pt]
  \scriptsize $\widehat{\bm\Phi}(\mathbf s)^{\!\top}\widehat{\bm\alpha}_{j}$
};
\node[refine] (refine) at (9.5,0.9) {
  \textbf{Prediction}\\[2pt]
  \scriptsize $\widehat{\mathbf Y}$
};
\node[stats] (stats) at (9.5,-0.9) {
  \textbf{Covariance and}\\
  \textbf{uncertainty}\\[2pt]
  \scriptsize $\widehat\Sigma_{\mathbf r}$
};
\draw[->] (backbone.east) -- (adapt.west);
\draw[->] (adapt.east)    -- (repr.west);
\draw[->] (repr.east)     -- (refine.west);
\draw[->] (repr.east)     -- (stats.west);
\end{tikzpicture}%
}
\caption{\textbf{Spatial Adapter overview.}
The adapter extracts an explicit low-rank spatial representation
$\widehat{\bm\Phi}(\mathbf s)^{\!\top}\widehat{\bm\alpha}_{j}$
from the residual field of a frozen first-stage predictor,
supporting prediction together with closed-form covariance
estimation and downstream uncertainty quantification.
Stage-2 ADMM details: Algorithm~\ref{alg:admm}.}
\label{fig:method-overview}
\end{figure}
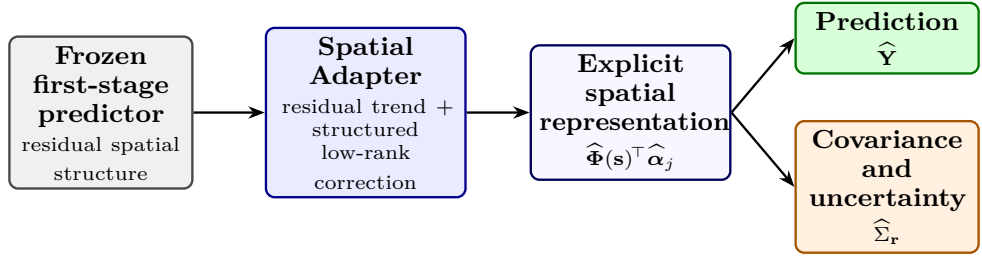

\subsection{Problem setup}\label{subsec:setup}

We observe a response matrix $\mathbf Y\in\mathcal Y^{T\times N}$
on a discrete spatial index set
$\mathcal S=\{\mathbf s_{i}\}_{i=1}^{N}$ over repetition indices
$j=1,\dots,T$, with $\mathcal Y=\mathbb R$ for regression and
$\mathcal Y=\{0,1\}$ for binary classification.  No assumption is
made on the geometry of $\mathcal S$: $\mathbf s_{i}$ may be a
geographic coordinate (with $j$ a time step) or a patch position on
a regular image grid (with $j$ an image index), and the formulation
that follows covers both interpretations.  Each entry carries a
covariate vector $\mathbf x_{j,i}\in\mathbb R^{p}$, gathered into a
tensor $\mathbf X\in\mathbb R^{T\times N\times p}$.  Throughout,
$g:\mathbb R\!\to\!\mathbb R$ denotes a fixed, monotone link
function applied entry-wise---the identity for regression and the
sigmoid for binary classification.

\textbf{Scope.}  The adapter acts only on the \emph{spatial} axis
$\mathbf s$: temporal or cross-sample dependence on $j$ is
delegated to the first-stage predictor, and Stage~2 places no
prior on $\bm\alpha_{j}$ across $j$.  When this dependence is not
absorbed by the first stage, the working-independence
$\widehat\Lambda$ may understate the marginal second moment of
$\bm\alpha_{j}$ and propagate into the plug-in variance; pairing
with temporally aware backbones mitigates this in practice.

\subsection{Generalized latent-field model}\label{subsec:model}

The response is modelled as a GLM with latent field
$\bm\eta\in\mathbb R^{T\times N}$,
$\mathbf Y_{j,i}\sim p(\cdot\mid\bm\eta_{j,i})$,
$\mathbb E[\mathbf Y_{j,i}\mid\bm\eta_{j,i}]=g(\bm\eta_{j,i})$
(Gaussian/identity for regression, Bernoulli/sigmoid for binary),
with additive decomposition
\begin{equation}\label{eq:latent-decomp}
\bm\eta_{j,i}
\;=\;
\underbrace{f_{\bm\theta}(\mathbf x_{j,i})}_{\text{frozen first stage}}
\;+\;
\underbrace{\mu_{\bm\psi}(\mathbf x_{j,i})}_{\text{residual trend}}
\;+\;
\underbrace{\sum_{k=1}^{K}\alpha_{k,j}\,\phi_{k}(\mathbf s_{i})}_{\text{low-rank spatial}}.
\end{equation}
$f_{\bm\theta}$ is any frozen predictor (deep model or linear fit);
$\mu_{\bm\psi}$ is a compact residual trend, zero-initialised and
refined in Stage~2; the basis
$\bm\Phi=[\phi_{k}(\mathbf s_{i})]_{i,k}\in\mathbb R^{N\times K}$
is regularised by a thin-plate-spline (TPS) roughness matrix
$\bm\Omega$~\citep{green1993nonparametric,wood2017generalized},
whose continuous-domain smoothness lets $\widehat{\bm\Phi}$ be
post-hoc extended to query points
$\mathbf s^{*}\!\notin\!\{\mathbf s_{i}\}$ via TPS interpolation
(used by the kriging predictor of Section~\ref{subsec:cov-est}).
Scores $\bm\alpha_{j}\!\in\!\mathbb R^{K}$ stack into
$\mathbf A\!\in\!\mathbb R^{T\times K}$; Stage~2 places no prior
on $\mathbf A$ across $j$, and the covariance estimator of
Section~\ref{subsec:cov-est} uses only its second moment
(\emph{working-independence}; residual temporal or cross-sample
dependence is delegated to the first stage).  In matrix form,
\begin{equation}\label{eq:matrix-H}
\bm\eta
\;=\;
\mathbf F_{\bm\theta}+\mathbf M_{\bm\psi}+\mathbf A\,\bm\Phi^{\!\top}.
\end{equation}

\subsection{Two-stage training objective}\label{subsec:twostage}

\paragraph{Stage 1 (warm-up).}
With $f_{\bm\theta}$ frozen, we pretrain the residual-trend network
$\mu_{\bm\psi}$ by minimizing the task loss
$\mathcal L_{\mathrm{task}}(g(f_{\bm\theta}+\mu_{\bm\psi}),\mathbf Y)$
(MSE for regression, BCE for binary classification); no spatial
basis is used yet.  Stage 1 supplies a warm start that Stage 2
refines inside the ADMM loop.

\paragraph{Stage 2 (adapter).}
Define the natural-parameter residual
\begin{equation}\label{eq:residual-R}
  \mathbf R(\bm\psi)
  \;\coloneqq\;
  g^{\dagger}(\mathbf Y_{\varepsilon})
  - \mathbf F_{\bm\theta}
  - \mathbf M_{\bm\psi}(\mathbf X)
  \;\in\;\mathbb R^{T\times N},
\end{equation}
where $g^{\dagger}$ is the canonical inverse link and
$\mathbf Y_{\varepsilon}$ clamps binary labels to
$[\varepsilon,1-\varepsilon]$ to keep the logit finite.
For Bernoulli outputs, $g^{\dagger}(\mathbf Y_{\varepsilon})$ is a
link-scale \emph{pseudo-response} rather than a direct observation
of the latent field; Stage~2 is accordingly interpreted as a
surrogate-based residual adaptation on the natural-parameter
scale, not an exact latent-variable likelihood decomposition.
Under~\eqref{eq:latent-decomp} the residual should lie in the span
of $\bm\Phi$---i.e.\ its off-basis component
$\mathbf R\mathbf P^{\!\perp}$
($\mathbf P=\bm\Phi\bm\Phi^{\!\top}$,
$\mathbf P^{\!\perp}=\mathbf I_{N}-\mathbf P$) should be small.  We
fit $(\bm\psi,\bm\Phi)$ by solving
\begin{equation}\label{eq:loss-general}
  \min_{\bm\psi,\,\bm\Phi}\;
    \ell_{\mathrm{data}}\!\bigl(\mathbf R(\bm\psi)\,\mathbf P^{\!\perp}\bigr)
    +\lambda_{1}\operatorname{tr}\!\bigl(\bm\Phi^{\!\top}\bm\Omega\,\bm\Phi\bigr)
    +\lambda_{2}\|\bm\Phi\|_{1}
  \quad\text{s.t.}\quad
    \bm\Phi^{\!\top}\bm\Phi=\mathbf I_{K},
\end{equation}
with $\ell_{\mathrm{data}}(\mathbf M)=\|\mathbf M\|_{F}^{2}$ in the
Gaussian/identity case (admitting closed-form Z-update and
eigendecomposition-based basis update; trend step is a gradient
step regardless) and
$\ell_{\mathrm{data}}=\sum_{j,i}\mathrm{BCE}(\sigma(\mathbf M_{j,i}),\mathbf Y_{j,i})$
in the Bernoulli case (handled by the surrogate of
Remark~\ref{rem:bce-basis-step}).  The roughness penalty
$\lambda_{1}$ imposes spatial smoothness, $\lambda_{2}$ promotes
local support, and the Stiefel constraint identifies $\bm\Phi$ up
to per-column sign.  Hyperparameters $(\lambda_{1},\lambda_{2})$
are tuned via Optuna~\citep{akiba2019optuna} (default objective
$\ell_{\text{val}}=\mathrm{RMSE}$;
Section~\ref{subsec:kaust} ablates CovFrob and a semivariogram
score).  The rank $K$ is chosen as the smallest value whose
top-$K$ eigenvalues of $\mathbf S=\tfrac{1}{T}\mathbf R^{\!\top}\mathbf R$
capture a fraction $\tau_{\mathrm{var}}=0.9$ of its total variance,
\begin{equation}\label{eq:rank-selection}
  K
  \;=\;
  \min\!\bigl\{k:\;\textstyle\sum_{i=1}^{k}\hat\lambda_{i}
               /\sum_{i=1}^{N}\hat\lambda_{i}\ge\tau_{\mathrm{var}}\bigr\};
\end{equation}
The basis dimension $K$ in~\eqref{eq:rank-selection} is only an
\emph{operational upper bound}.  The low-rankness of the
downstream covariance model is determined by the retained spectral
directions in
$\widehat\Sigma_{\mathbf r}=\widehat{\bm\Phi}\,\widehat\Lambda\,\widehat{\bm\Phi}^{\!\top}
+\widehat\sigma^{2}\mathbf I$~\eqref{eq:cov-est}, with weak
components further truncated by the
Wang--Huang~\citep{wang2017regularized} eigenvalue-thresholding
rule (Eq.~\eqref{eq:wang-rank}, Appendix~\ref{app:cov-correction}).
Thus covariance estimation and rank determination are tightly
coupled: the \emph{effective retained rank}---i.e.\ the number
of spectral components surviving the eigenvalue-thresholding
rule---is a property of the covariance estimator, not of $K$
alone.  The Stiefel constraint
$\bm\Phi^{\!\top}\bm\Phi=\mathbf I_{K}$ is essential here: without
orthogonality, the factorization
$\widehat{\bm\Phi}\,\widehat\Lambda\,\widehat{\bm\Phi}^{\!\top}$
is identified only up to an arbitrary rotation, and thresholding
its spectrum would not select stable directions.  Orthogonality
fixes the columns up to per-column sign; the roughness
($\lambda_{1}$) and column-sparsity ($\lambda_{2}$) penalties then
shape these identifiable directions to be smooth and locally
supported, so the effective retained rank measures how many
\emph{structured} spatial components survive thresholding---not
merely a nominal low-rank dimension.  Per-experiment
cumulative-variance curves and chosen-$K$ vs.\ effective-rank
numbers are in Appendix~\ref{app:rank-selection}.

We decouple $\bm\psi$ and $\bm\Phi$ in~\eqref{eq:loss-general} via a
consensus variable $\mathbf Z$ standing in for $\mathbf R(\bm\psi)$:
\begin{equation}\label{eq:loss-admm}
\min_{\bm\psi,\bm\Phi,\mathbf Z}\;
  \ell_{\mathrm{data}}(\mathbf Z\mathbf P^{\!\perp})
  +\lambda_{1}\operatorname{tr}(\bm\Phi^{\!\top}\bm\Omega\bm\Phi)
  +\lambda_{2}\|\bm\Phi\|_{1},
\quad
\text{s.t.}\;
  \bm\Phi^{\!\top}\bm\Phi=\mathbf I_{K},\;
  \mathbf Z=\mathbf R(\bm\psi).
\end{equation}
Per-sample scores are not explicit primal variables; at convergence
$\widehat{\mathbf A}=\widehat{\mathbf Z}\,\widehat{\bm\Phi}$
(regularization-path figures: Appendix~\ref{app:reg-path}).

\subsection{Optimization via mini-batch ADMM}\label{subsec:admm}

We solve~\eqref{eq:loss-admm} by a stochastic three-block
ADMM~\citep{boyd2011distributed} over primal variables
$(\bm\psi,\bm\Phi,\mathbf Z)$ and scaled dual $\mathbf U$ (for the
constraint $\mathbf Z=\mathbf R(\bm\psi)$).  At each step we draw a
mini-batch $\mathcal B\subset\{1,\dots,T\}$ of size $L$, maintain
$\mathbf Z,\mathbf U\!\in\!\mathbb R^{L\times N}$ on the batch, and
minimize the augmented Lagrangian
$\mathcal L_{\rho}=\ell_{\mathrm{data}}(\mathbf Z\mathbf P^{\!\perp})
+\lambda_{1}\operatorname{tr}(\bm\Phi^{\!\top}\bm\Omega\bm\Phi)
+\lambda_{2}\|\bm\Phi\|_{1}
+\tfrac{\rho}{2}\|\mathbf Z-\mathbf R(\bm\psi)+\mathbf U\|_{F}^{2}$
by alternating updates (T)$\to$(B)$\to$(Z)$\to$(U) as in
Algorithm~\ref{alg:admm}.

\begin{algorithm}[t]
\caption{Spatial Adapter: Stage~2 mini-batch ADMM.}
\label{alg:admm}
\begin{algorithmic}[1]
\Require frozen first-stage predictor $f_{\hat{\bm\theta}}$; warm-started
         trend $\bm\psi$; regularization weights
         $(\lambda_{1},\lambda_{2})$; penalty $\rho$; rank $K$;
         mini-batch size $L$; trend learning rate $\eta_{\psi}$
\State Initialize orthonormal $\bm\Phi\in\mathbb R^{N\times K}$,
       $\mathbf Z\gets\mathbf R(\bm\psi)$, $\mathbf U\gets\mathbf 0$
\Repeat
  \State Draw a mini-batch $\mathcal B\subset\{1,\dots,T\}$ of size $L$
  \State \textbf{(T) Trend step:}\hspace{0.4em}
         update $\bm\psi$ by one gradient step on
         $\tfrac{\rho}{2}\|\mathbf Z_{\mathcal B}-\mathbf R_{\mathcal B}(\bm\psi)+\mathbf U_{\mathcal B}\|_{F}^{2}$
  \State \textbf{(B) Basis step:}\hspace{0.4em}
         update $\bm\Phi$ from $\mathbf Z_{\mathcal B}$
         \Comment{closed-form eigendecomposition (Gaussian);
                  surrogate update (Bernoulli)}
  \State \textbf{(Z) Consensus step:}\hspace{0.4em}
         update $\mathbf Z_{\mathcal B}$
         \Comment{closed form~\eqref{eq:z-closed-form} (Gaussian);
                  one prox-gradient step (Bernoulli)}
  \State \textbf{(U) Dual step:}\hspace{0.4em}
         $\mathbf U_{\mathcal B}\gets\mathbf U_{\mathcal B}+\mathbf Z_{\mathcal B}-\mathbf R_{\mathcal B}(\bm\psi)$
  \State Optionally freeze $\bm\Phi$ after a fixed number of
         iterations
\Until{validation loss stops improving}
\State \Return $\hat{\bm\psi},\,\widehat{\bm\Phi}$, and in-sample
       scores $\widehat{\mathbf A}=\widehat{\mathbf Z}\,\widehat{\bm\Phi}$
\end{algorithmic}
\end{algorithm}

\paragraph{Interpretation.}
Each sweep is a computational analogue of the classical spatial
decomposition into trend, low-rank spatial effect, and noise:
(T) refines the trend, (B) updates the basis, (Z) in the
Gaussian/identity case retains the in-basis residual and shrinks
the off-basis component (Proposition~\ref{prop:z-update}),
and (U) enforces consistency.  Step (B) is a top-$K$
eigendecomposition of a symmetrized penalized Gram matrix (with an
IRL$_{1}$~\citep{candes2008enhancing} inner loop when
$\lambda_{2}>0$); step (Z) is closed-form in the Gaussian case
and one proximal-gradient step in the Bernoulli case; after a
freeze horizon $N_{\text{freeze}}$, (B) is skipped.  Full
derivations are in Appendix~\ref{app:admm-steps}.

\begin{remark}[Basis update in the Bernoulli case]
\label{rem:bce-basis-step}
The $\bm\Phi$-subproblem under the Bernoulli data term admits no
closed form.  We apply the standard $\tfrac{1}{8}$-quadratic
surrogate for the log-sigmoid~\citep{bohning1992multinomial,jaakkola2000bayesian}
and a variance-only target $\mathbf C_{\mathcal B}$ (dropping the
label-dependent correction from the full Taylor expansion).
The full surrogate derivation, entrywise remainder bound, and an
empirical ablation against two alternative targets are given in
Appendix~\ref{app:bce-ablation}.
\end{remark}

In the Gaussian/identity case, the $\mathbf Z$-subproblem is
$\rho$-strongly convex and admits the closed form
\begin{equation}\label{eq:z-closed-form}
\mathbf Z_{\mathcal B}^{\star}
=
\mathbf{Res}_{\mathcal B}\,\mathbf P
+\frac{\rho}{\rho+2}\,\mathbf{Res}_{\mathcal B}\,\mathbf P^{\!\perp},
\qquad
\mathbf{Res}_{\mathcal B}\coloneqq\mathbf R_{\mathcal B}(\bm\psi)-\mathbf U_{\mathcal B},
\end{equation}
which retains the in-basis residual and shrinks its off-basis
component by $\rho/(\rho+2)\!\in\!(0,1)$ (Proposition~\ref{prop:z-update},
Appendix~\ref{app:proofs}; $\rho$ chosen so that
$\rho/(\rho+2)\!\approx\!0.5$).  We make no global convergence
claims for this nonconvex stochastic three-block ADMM;
stabilization in practice comes from the Stage-1 warm start, the
strongly convex Gaussian $\mathbf Z$-subproblem, and infrequent,
then frozen, basis updates.  The method targets moderate-$N$
discrete spatial grids: at $N\!\le\!2000$ the basis step
($O(LN^{2}{+}N^{3})$) does not bottleneck under our cadence/freeze
schedule, while larger $N$ would require implicit-operator or
randomized-SVD variants in place of explicit $N\!\times\!N$ Gram
materialisation (Appendix~\ref{app:periter-cost}).

\subsection{Spatial covariance estimation}\label{subsec:cov-est}

The basis $\widehat{\bm\Phi}$ from Stage~2 immediately yields a
closed-form estimate of the spatial covariance of the residual
field, unlocking two capabilities beyond point prediction:
spatial interpolation at unobserved locations (underlying the
Weather2K spatial-holdout experiment), and plug-in predictive
intervals (illustrated diagnostically in
Appendix~\ref{app:bce-ablation}).  Under the
working-independence model
$\bm\alpha_{j}\sim(\mathbf 0,\Lambda)$,
$\bm\epsilon_{j}\sim(\mathbf 0,\sigma^{2}\mathbf I_{N})$ of
Section~\ref{subsec:model}, the residual
$\mathbf r_{j}=\mathbf Y_{j,:}-(\mathbf F_{\hat{\bm\theta}})_{j,:}-(\mathbf M_{\hat{\bm\psi}})_{j,:}=\bm\Phi\bm\alpha_{j}+\bm\epsilon_{j}$
has marginal spatial covariance
\begin{equation}\label{eq:cov-decomp}
  \Sigma_{\mathbf r}
  \;=\;
  \bm\Phi\,\Lambda\,\bm\Phi^{\!\top}
  +\sigma^{2}\mathbf I_{N}.
\end{equation}
Following Proposition~1
of~\citet{wang2017regularized},\footnote{We use the closed-form
moment estimator of~\citet{wang2017regularized}, with a small
clarification of the rank-selection rule (their Eq.~(12)) that
handles the corner case where the defining set is empty; details
in Appendix~\ref{app:cov-correction}.}
the method-of-moments plug-in estimator is
\begin{equation}\label{eq:cov-est}
\widehat\Sigma_{\mathbf r}
\;=\;
\widehat{\bm\Phi}\,\widehat\Lambda\,\widehat{\bm\Phi}^{\!\top}
+\widehat\sigma^{2}\,\mathbf I_{N},
\end{equation}
with $(\widehat\Lambda,\widehat\sigma^{2})$ in closed form from
the eigenvalues of $\widehat{\bm\Phi}^{\!\top}\mathbf S\widehat{\bm\Phi}$
($\mathbf S\coloneqq\tfrac{1}{T_{\mathrm{train}}}\mathbf R^{\!\top}\mathbf R$);
costs a single $K\!\times\!K$ eigendecomposition and is unbiased
for $\Sigma_{\mathbf r}$ when $\widehat{\bm\Phi}=\bm\Phi$.

The estimated basis $\widehat{\bm\Phi}$ is extended to query
locations $\mathbf s^{*}$ outside the Stage-2 grid via a
deterministic thin-plate-spline interpolator (no extra training),
and~\eqref{eq:cov-est} drives two plug-in predictive machineries
derived in Appendix~\ref{app:conditional-variance}:
a fixed-rank kriging conditional predictive
law~\citep{cressie2008fixed} that gives the $100(1-\alpha)\%$
regression interval
$\hat\eta(\mathbf s^{*},j)\pm z_{\alpha/2}\sqrt{\hat v(\mathbf s^{*},j)}$~\eqref{eq:cgi}
(Proposition~\ref{prop:gaussian-predictive}), and its
logistic-normal pushforward through $\sigma$ that gives the binary
probability-scale interval~\eqref{eq:lnui}.  Both are asymptotic
plug-ins; the binary interval construction is illustrated
diagnostically in Appendix~\ref{app:bce-ablation} (variance-only
surrogate, fixed $K{=}16$), with regression-side calibration on
Weather2K deferred to future work (bias sources: Limitations,
Section~\ref{sec:discussion}).  Stage~2 recovers regularized
spatial PCA~\citep{wang2017regularized} as the $\mu_{\bm\psi}=0$,
$f_{\bm\theta}\!\equiv\!0$, $g=\mathrm{identity}$, $\lambda_{2}=0$,
full-batch limit.

\section{Numerical Study}\label{sec:numerical}

This section validates the adapter under controlled conditions.
Together with the real-world deployments of
Section~\ref{sec:applications}, the four experiments in this
paper span four orthogonal axes so that each answers a distinct
research question rather than repeating the same test in a new
dataset:
\begin{itemize}[leftmargin=1.3em,itemsep=1pt]
  \item \textbf{Task type}: continuous regression
        (\S\ref{subsec:kaust}, \S\ref{subsec:weather2k}) vs.\
        binary image classification (\S\ref{subsec:image}).
  \item \textbf{Backbone strength}: deliberately under-trained
        STDK (\S\ref{subsec:kaust}), well-trained STDK
        (\S\ref{subsec:weather2k}), and large pretrained vision
        backbones (\S\ref{subsec:image}).
  \item \textbf{Spatial index $\mathbf s$}: geographic station
        network (\S\ref{subsec:kaust}, \S\ref{subsec:weather2k})
        vs.\ within-image patch grid with genuine intra-image
        spatial dependence (\S\ref{subsec:image}).
  \item \textbf{Evaluation angle}: validation-criterion ablation
        (\S\ref{subsec:kaust}); kriging-based spatial prediction at
        $20\%$ held-out stations
        (\S\ref{subsec:weather2k}, with a complementary data-split
        ablation in Appendix~\ref{app:weather2k-split}); portability
        to binary patch-level labels and the BCE variant
        (\S\ref{subsec:image}).
\end{itemize}
A compact synthetic mechanism validation (§\ref{subsec:synth-main})
precedes the KAUST ablation; the full DGP and extended metrics
are in Appendix~\ref{subsec:synth}.  The two real-world applications
follow in Section~\ref{sec:applications}.

\subsection{Evaluation metrics}\label{subsec:metrics}

The evaluation protocol is shared by Sections~\ref{sec:numerical}
and~\ref{sec:applications}.  All metrics report mean and standard
error across seeds (count stated in each table caption).  We use
the following metrics, with full definitions deferred to
Appendix~\ref{app:metrics}:
\begin{itemize}[leftmargin=1.3em,itemsep=1pt]
  \item \textbf{Pointwise metrics.}  RMSE/MAE/$R^{2}$ (regression,
        Eq.~\eqref{eq:pointwise_def}), Acc/AUC/$F_{1}$ (binary).
        Under time-split and within-image protocols $\mathbf Y$ is
        observed on the eval set, so adapter rows read as
        \emph{reconstruction} diagnostics, not predictive
        classification/forecasting; genuine out-of-sample prediction
        is Weather2K held-out (Section~\ref{subsec:weather2k}).
  \item \textbf{Spatial-dependence fidelity.}  Relative Frobenius
        covariance error \textbf{CovFrob}
        (Eq.~\eqref{eq:covfrob_def}) and a semivariogram score
        $\mathbf{SV_{\text{score}}}$
        (Eq.~\eqref{eq:semivar-score}); reported in regression
        settings where the sample covariance is well-conditioned.
  \item \textbf{Interval diagnostics (appendix).}  The binary
        logistic-normal interval is illustrated with MPIW (on the
        probability scale) and a location-level $\mathbf{CP}$
        against the cross-image empirical occurrence rate
        $\bar p(\mathbf s)$ in Appendix~\ref{app:bce-ablation};
        type-match rationale and full metric definitions in
        Appendix~\ref{app:metrics}.  These metrics are not part of
        the main-text empirical claims.
\end{itemize}

\subsection{Synthetic mechanism validation}\label{subsec:synth-main}

On a rank-1 synthetic DGP with known $\phi(\mathbf s)$ and
$\Sigma_{\mathrm{true}}$ (full DGP in
Appendix~\ref{subsec:synth}), Figure~\ref{fig:synth} contrasts
what a frozen OLS first stage vs.\ the adapter recover.  OLS is
predictively adequate (RMSE within $0.2\%$ of the regularized
adapter; Table~\ref{tab:temporal_results}) but structurally
opaque: its residual sample covariance buries the Gaussian-bump
structure of $\Sigma_{\mathrm{true}}$ under full-rank sampling
noise.  The regularized adapter panel recovers that structure,
and the path plot shows basis alignment and CovFrob concentrating
on a common optimum under $\lambda_{1}\!=\!\lambda_{2}\!=\!\lambda$
--- explanatory power added without sacrificing pointwise
reconstruction quality.

\begin{figure}[!ht]
  \centering
  \begin{minipage}[c]{0.52\linewidth}
    \includegraphics[width=\linewidth]{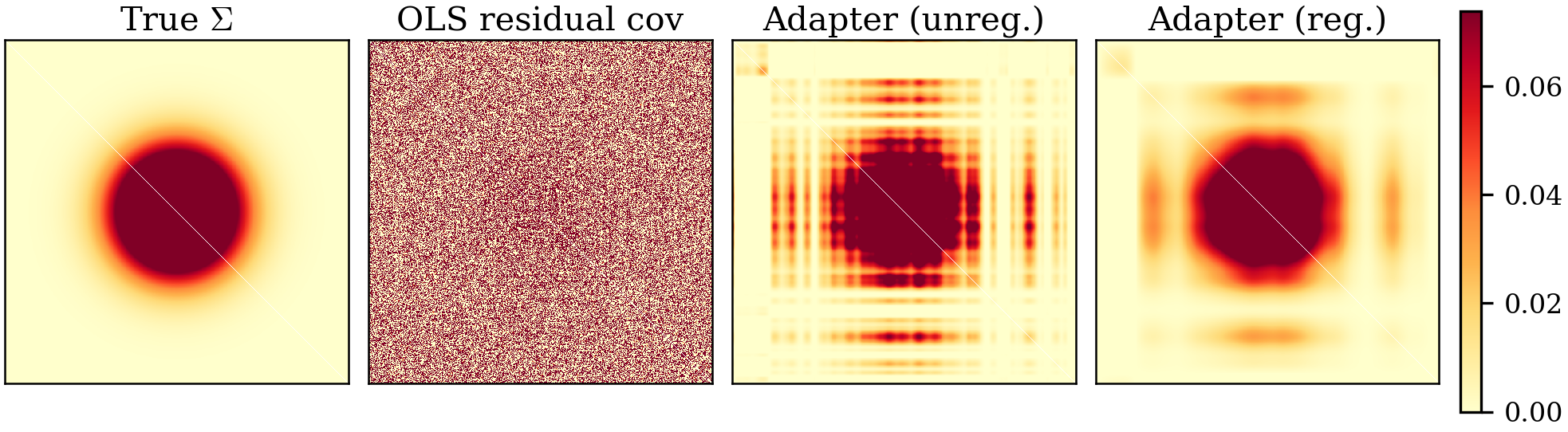}
  \end{minipage}\hfill
  \begin{minipage}[c]{0.46\linewidth}
    \includegraphics[width=\linewidth]{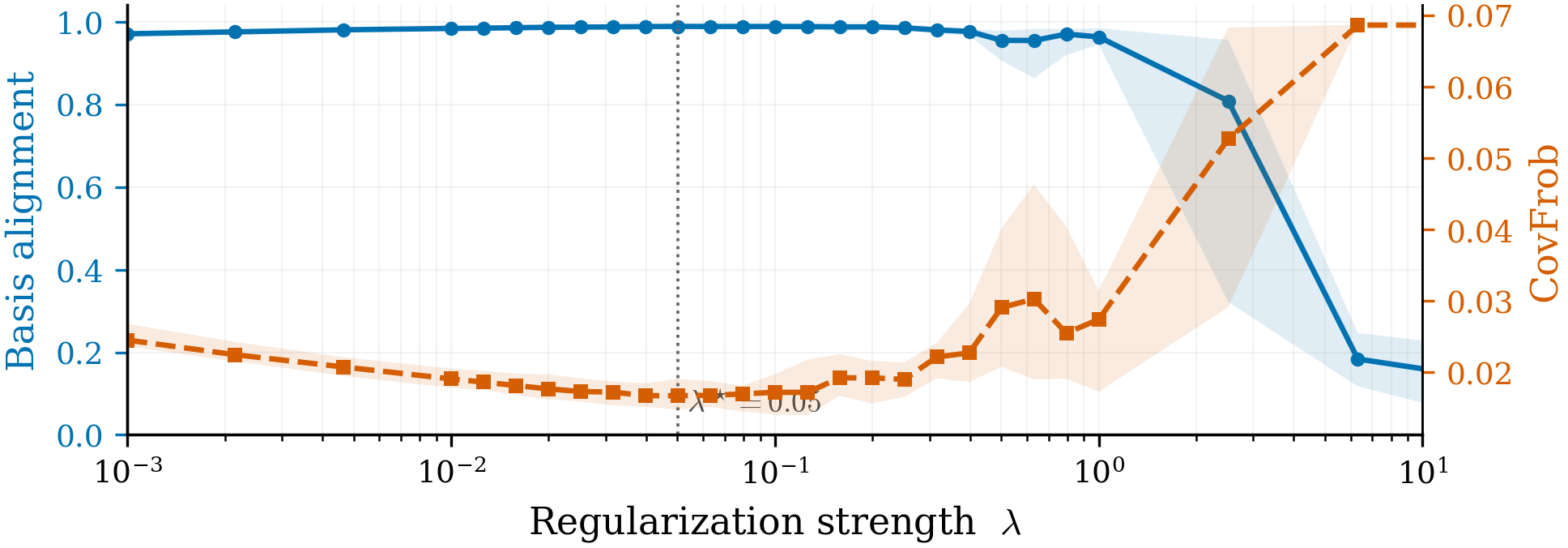}
  \end{minipage}
  \caption{\textbf{Synthetic mechanism validation (rank-1 DGP,
           representative seed).}
           \textbf{Left:} $\Sigma_{\mathrm{true}}$ / OLS residual
           cov / adapter at $\lambda{=}10^{-3}$ (``unreg.'') /
           adapter at tuned $\lambda{=}0.05$ (``reg.'');
           shared off-diagonal colormap, diagonal masked.
           CovFrob vs.\ $\Sigma_{\mathrm{true}}$: OLS $0.847$,
           unreg.\ $0.025$, reg.\ $0.018$.
           \textbf{Right:} regularization path
           ($\lambda_{1}=\lambda_{2}=\lambda$) for basis alignment
           (blue) and CovFrob (orange), median and IQR over 30
           seeds; dotted line marks $\lambda^{\star}$.  Multi-seed
           numerics in Table~\ref{tab:temporal_results}
           (Appendix~\ref{subsec:synth}).}
  \label{fig:synth}
\end{figure}

\subsection{KAUST Competition Data --- Ablation on Validation Criterion}
\label{subsec:kaust}

We pool the KAUST 2b\_8 train and test releases into a single
grid of $N_{\text{full}}{=}10^{4}$ locations over $T_{\text{full}}{=}100$
time steps, retain $10\%$ of stations per seed ($N{=}1000$), and
split time $10/10/80$ into train/val/test.  The backbone is
STDK~\citep{lin2023enhancements} trained on the $10\%$ train window;
the adapter uses $K{=}46$ with $50$ Optuna trials over
$(\lambda_{1},\lambda_{2})\!\in\![10^{-4},10^{8}]^{2}$.  We compare
three validation criteria---RMSE, CovFrob~\eqref{eq:covfrob_def},
and SV$_{\mathrm{score}}$~\eqref{eq:semivar-score}---and for each
refit the adapter at the selected
$(\hat\lambda_{1},\hat\lambda_{2})$ and evaluate all three metrics
on the test set (setup details in Appendix~\ref{app:kaust-setup}).

\begin{table}[htbp]
  \centering
  \caption{Ablation on validation criterion (KAUST data, 10/10/80
           split).  All metrics are reconstruction-based (the adapter
           observes $\mathbf Y$ at all stations); evaluated on test
           set; mean (std.\ error) across 30 seeds.}
  \label{tab:ablation_val}
  \begin{tabular}{@{}lccc@{}}
    \toprule
    \textbf{Val criterion}
      & \textbf{Test RMSE}
      & \textbf{Test Cov\(_{\text{Frob}}\)}
      & \textbf{Test SV\(_{\text{score}}\)} \\
    \midrule
    Backbone only (no adapter)
      & 1.043\,(0.004)
      & 0.995\,(0.000)
      & 0.466\,(0.007) \\
    RMSE
      & 0.465\,(0.002)
      & 0.382\,(0.001)
      & 0.038\,(0.001) \\
    CovFrob
      & 0.465\,(0.002)
      & 0.382\,(0.001)
      & 0.037\,(0.001) \\
    SV\(_{\text{score}}\)
      & 0.465\,(0.002)
      & 0.382\,(0.001)
      & 0.037\,(0.001) \\
    \bottomrule
  \end{tabular}
\end{table}

This ablation addresses two questions.
First, \emph{robustness}: when training data are scarce, the adapter
must separate genuine spatial signal from trend-fitting artifacts.
Second, \emph{trade-off}: RMSE-selected regularization may favor
pointwise reconstruction at the expense of covariance fidelity,
whereas CovFrob or semivariogram selection may sacrifice marginal
pointwise accuracy to better recover spatial dependence.
Notably, all three validation criteria converge on near-identical
test metrics (test RMSE $\approx 0.465$, Cov$_{\mathrm{Frob}}\approx 0.382$,
SV$_{\text{score}}\approx 0.037$ within one SE).  This is
consistent with the deliberately under-trained backbone: when the
first-stage trend is weak, a large fraction of the spatial signal
remains in the residual, so that pointwise reconstruction and
covariance recovery are tightly coupled---improving one
automatically improves the other.  In better-trained regimes
(e.g.\ the 80/10/10 Weather2K split of
Section~\ref{subsec:weather2k}), this coupling is expected to
weaken, and a covariance-aware validation criterion may become more
informative.

\section{Applications}\label{sec:applications}

Two real-world settings span the two interpretations of the spatial
index $\mathbf s$: geographic weather stations with continuous
response (Weather2K, \S\ref{subsec:weather2k}) and a within-image
patch grid with binary response (GWHD Wheat Head,
\S\ref{subsec:image}).

\subsection{Weather2K --- Held-Out Spatial Prediction}\label{subsec:weather2k}

We use Weather2K ($\sim\!2000$ Chinese weather stations), subsampled to
$N{=}187$ stations and $T{=}1000$ time steps.  An
STDK backbone~\citep{lin2023enhancements} is trained on the
training split and the adapter is applied with $K{=}13$.
STDK is itself a classical spatiotemporal backbone with built-in
spatial structure; applying the adapter on top tests whether
residual spatial patterns persist even when the first stage is
already spatially aware---i.e.\ whether the adapter is a
\emph{complementary} layer rather than a remedy for weak
backbones.
$(\lambda_{1},\lambda_{2})$ are selected by Optuna over 50 trials
on validation RMSE.  We compare STDK alone, adapter (unreg.,
$\lambda_{1}=\lambda_{2}=0$), and adapter (reg.).  A complementary
data-split ablation is in Appendix~\ref{app:weather2k-split}.

\paragraph{Spatial prediction at held-out stations.}
We hold out $20\%$ of stations entirely from training (spatially
masked in both the STDK backbone and the adapter) and apply the
conditional-kriging predictor of Section~\ref{subsec:cov-est}:
$\widehat{\bm\phi}(\mathbf s^{*})$ is obtained by thin-plate-spline
interpolation of the observed-station basis rows, and
$\widehat{\bm\alpha}_{j}(\mathcal O_{j})$ from
Eq.~\eqref{eq:point-pred} (Appendix~\ref{app:conditional-variance})
uses the residuals at training stations at test time $j$.  We
report held-out RMSE; the closed-form predictive
variance~\eqref{eq:kriging-se} (same appendix) is not evaluated as a PI diagnostic
here (conformal calibration: Section~\ref{sec:discussion}).
This is a \emph{conditional spatial interpolation} setting
(adapter uses same-time observed-station residuals to predict
held-out stations; the STDK baseline extrapolates without such
conditioning).  The regularized adapter reduces RMSE to
$2.664~(0.015)$ vs.\ unreg.\ $2.705~(0.017)$ and STDK alone
$7.667~(0.034)$; full numbers in
Table~\ref{tab:weather2k_holdout}
(Appendix~\ref{app:weather2k-held-out-extra}).

The reduction evidences that the learned basis $\widehat{\bm\Phi}$
transports to unseen stations through thin-plate-spline
interpolation (Fig.~\ref{fig:weather2k_holdout_main}).  The three
panels look visually similar at the continental scale --- STDK
already captures the dominant temperature trend, so the adapter's
gain concentrates on local residual structure; the station-split
map and per-station RMSE in
Appendix~\ref{app:weather2k-held-out-extra} resolve that
local-scale improvement.  The experiment is
scoped to held-out RMSE to isolate kriging transportability;
regression-side interval calibration uses the same covariance
machinery and is left to future work
(Section~\ref{sec:discussion}).

\begin{figure*}[!tbp]
  \centering
  \includegraphics[width=\linewidth]{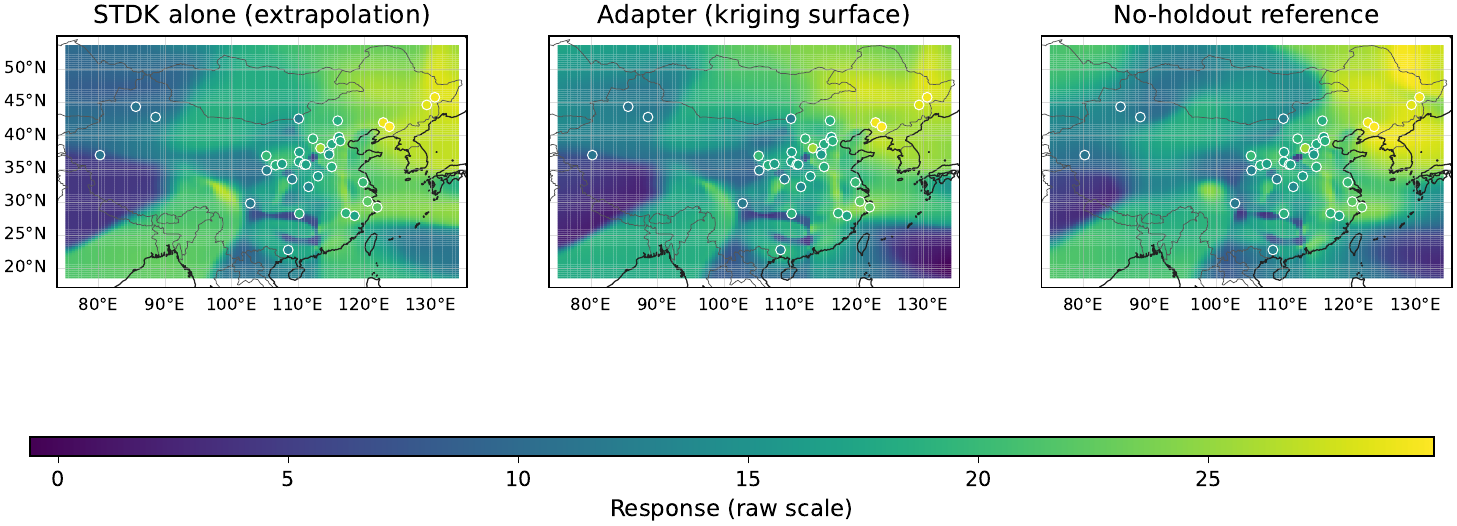}
  \caption{Held-out Weather2K kriging at one test time.
           \textbf{Left:} STDK extrapolation baseline.
           \textbf{Middle:} adapter kriging surface --- STDK plus
           the learned low-rank correction
           $\widehat{\bm\Phi}(\mathbf s)^{\!\top}\widehat{\bm\alpha}_{j}$,
           interpolated to held-out locations via thin-plate splines.
           \textbf{Right:} no-holdout reference --- the same pipeline
           refit with the $20\%$ held-out stations included in
           training, so the basis is read off in-sample (no TPS
           interpolation).  All three panels share the response colour
           scale and the $20\%$ held-out stations are overlaid on
           every panel.  The adapter kriging surface tracks the
           no-holdout reference closely at the continental scale; the multi-seed
           Adapter-vs-STDK RMSE gap is given in
           Table~\ref{tab:weather2k_holdout}
           (Appendix~\ref{app:weather2k-held-out-extra}).  The signed
           correction surface is broken out in
           Fig.~\ref{fig:weather2k_correction}.}
  \label{fig:weather2k_holdout_main}
\end{figure*}

\begin{figure}[!ht]
  \centering
  \includegraphics[width=\linewidth]{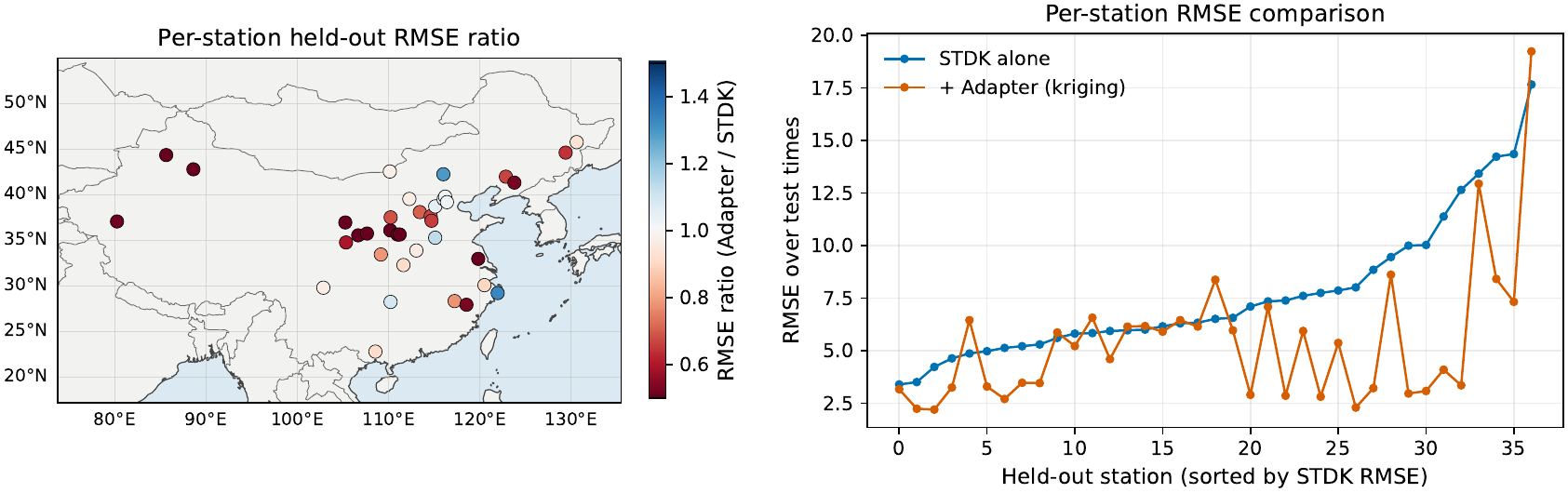}
  \caption{Per-station held-out RMSE on Weather2K.
           \textbf{Left:} RMSE ratio (adapter\,/\,STDK) overlaid
           on the held-out-station map; cool colours indicate the
           adapter outperforms the backbone.  \textbf{Right:}
           paired RMSE trajectories, stations sorted by STDK RMSE
           ---~the gap between the two curves shows consistent
           per-station improvement across the $20\%$ held-out
           stations, not an average driven by a few outliers.}
  \label{fig:weather2k_rmse}
\end{figure}

\subsection{Wheat Head: Basis Transferability on Patch Grids}\label{subsec:image}

We apply the adapter to within-image binary patch grids from the
Global Wheat Head Detection
dataset~\citep{david2020global,david2021global} as a
\emph{basis-transferability diagnostic}: at test time the adapter
projects each test image's observed patch labels onto the
train-learned basis $\widehat{\bm\Phi}$, so reported metrics
measure how well $\widehat{\bm\Phi}$ spans the test-time residual
structure rather than predictive classification (a label-efficient
partial-observation evaluation is future work).  GWHD supplies
$6{,}515$ images ($1024\!\times\!1024$) with $300{,}000\!+\!$
wheat-head annotations from $12$ countries.

\paragraph{Patch-level binary task.}
We convert the detection task into a patch-level binary
classification target used as the residual carrier.  Each image
is divided into a regular grid of non-overlapping patches (e.g.\
$64\!\times\!64$ pixels), yielding $N=16\!\times\!16=256$ patches
per image; a patch is labelled \emph{positive} if at least one
ground-truth bounding box overlaps it, \emph{negative} otherwise.
Images are split $70/15/15$ into train/val/test, with the same
split shared across all backbones; the adapter uses rank
$K\!\in\!\{153,155\}$ (per backbone, via the
cumulative-variance rule; see Appendix~\ref{app:rank-selection}).
All reported means and standard errors are across $30$ random
seeds of the adapter's Stage-2 training.

The patch grid maps onto the spatial GLM with
$\mathbf s=(h,w)\!\in\!\{1,\dots,16\}^{2}$, each image
indexing one sample $j$, and $Y(j,\mathbf s)\!\in\!\{0,1\}$ the
wheat-head presence.  We evaluate four frozen backbones---ResNet-152,
ConvNeXt-T, ViT-B/16, and SAM (ViT-H)~\citep{he2016deep,liu2022convnext,dosovitskiy2021vit,kirillov2023sam}---each
pretrained (ImageNet or SAM checkpoint) and coupled to a lightweight
classification head trained in Stage~1.  In Stage~2 the adapter is
applied with $g=\sigma$ and BCE loss, using the surrogate
updates of Remark~\ref{rem:bce-basis-step}; at convergence the
per-image reconstructed probability is
$\hat p(\mathbf s)=\sigma\bigl(f_{\hat{\bm\theta}}(\mathbf s)+\sum_{k}\alpha_{k,j}\phi_{k}(\mathbf s)\bigr)$,
with $\alpha_{k,j}$ recovered from the observed test-image patch
labels.
Full framework mapping, backbone details, the two-stage
recipe, and a parameter-efficiency breakdown showing that the
adapter adds $0.009$--$0.32\%$ of the frozen first-stage
parameter count
(Table~\ref{tab:wheat-param-efficiency})
are in Appendix~\ref{app:wheat-setup}.  Three comparison
variants are reported for each backbone:
\emph{backbone only} (no adapter),
\emph{+ adapter (unreg.)} ($\lambda_1\!=\!\lambda_2\!=\!0$), and
\emph{+ adapter (reg.)} with tuned $(\hat\lambda_1,\hat\lambda_2)$.

\begin{figure*}[!tbp]
  \centering
  \includegraphics[width=0.82\linewidth]{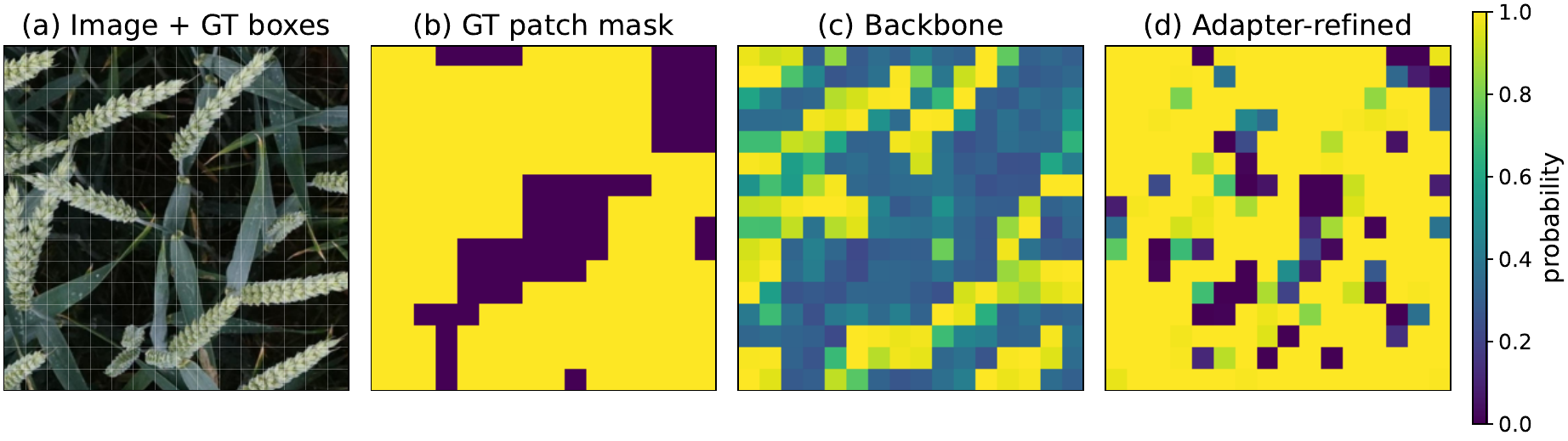}
  \caption{GWHD within-image refinement (ResNet-152, shared scale
           $[0,1]$): (a)~image + GT boxes, (b)~GT patch mask,
           (c)~backbone probability, (d)~adapter-refined
           (spatially sharpened, aligns with~(b)).}
  \label{fig:wheat_head_refinement}
\end{figure*}

\begin{table}[htbp]
  \centering
  \caption{\textbf{Basis transferability diagnostic on GWHD
           (reconstruction, not prediction).}  Adapter rows project
           each test image's observed patch labels onto the
           train-learned basis $\Phi$; high metrics measure how
           well $\Phi$ spans test-image residual structure, not
           predictive classification.  Backbone-only rows are the
           no-adapter prediction reference.  Mean (std.\ error)
           over $30$ seeds.}
  \label{tab:image}
  \begin{tabular}{@{}lccc@{}}
    \toprule
    \textbf{Backbone / model}
      & \textbf{Acc.\,(\%)}
      & \textbf{AUC}
      & \textbf{F1} \\
    \midrule
    \multicolumn{4}{@{}l}{\emph{Backbone-only prediction reference (no adapter):}} \\
    ResNet-152           & 78.134\,(0.082) & 0.86735\,(0.00033) & 0.800\,(0.002) \\
    ConvNeXt-T           & 81.089\,(0.038) & 0.89166\,(0.00032) & 0.825\,(0.001) \\
    ViT-B/16             & 82.623\,(0.037) & 0.90408\,(0.00028) & 0.840\,(0.001) \\
    SAM (ViT-H)          & 78.866\,(0.057) & 0.87533\,(0.00030) & 0.811\,(0.001) \\
    \midrule
    \multicolumn{4}{@{}l}{\emph{Adapter reconstruction diagnostic (observed test labels $\to$ projected onto $\Phi$):}} \\
    ResNet-152 + Adapter (unreg.)  & 96.164\,(0.046) & 0.99367\,(0.00014) & 0.966\,(0.000) \\
    ResNet-152 + Adapter (reg.)    & \textbf{99.472}\,(0.010) & \textbf{0.99986}\,(0.00001) & \textbf{0.995}\,(0.000) \\
    ConvNeXt-T + Adapter (unreg.)  & 96.865\,(0.024) & 0.99571\,(0.00005) & 0.973\,(0.000) \\
    ConvNeXt-T + Adapter (reg.)    & \textbf{99.481}\,(0.005) & \textbf{0.99987}\,(0.00000) & \textbf{0.995}\,(0.000) \\
    ViT-B/16 + Adapter (unreg.)    & 97.337\,(0.026) & 0.99682\,(0.00005) & 0.977\,(0.000) \\
    ViT-B/16 + Adapter (reg.)      & \textbf{99.584}\,(0.005) & \textbf{0.99991}\,(0.00000) & \textbf{0.996}\,(0.000) \\
    SAM (ViT-H) + Adapter (unreg.) & 95.762\,(0.042) & 0.99231\,(0.00014) & 0.963\,(0.000) \\
    SAM (ViT-H) + Adapter (reg.)   & \textbf{99.684}\,(0.003) & \textbf{0.99994}\,(0.00000) & \textbf{0.997}\,(0.000) \\
    \bottomrule
  \end{tabular}
\end{table}

\paragraph{Logistic-normal intervals (diagnostic).}
The same covariance machinery supports a link-scale
logistic-normal interval on each patch's probability;
because~\eqref{eq:lnui} quantifies uncertainty about
$p(\mathbf s)$ rather than the binary outcome
$y\!\in\!\{0,1\}$, calibration is reported as a location-level
statistic rather than a per-instance coverage.  We illustrate this
construction diagnostically (variance-only Bernoulli surrogate,
fixed $K{=}16$) in Appendix~\ref{app:bce-ablation} alongside the
full BCE-surrogate ablation; it is not part of the main empirical
claims of this section.

Three takeaways: (i) $\mathbf s$ generalises from geographic
stations (Weather2K) to within-image patch grids (Wheat Head)
with no method-side change; (ii) the train-learned basis
transfers across test images on all four frozen vision backbones
(near-perfect residual reconstruction); (iii)
$\widehat\Sigma_{\mathbf r}$ underlies both genuine
spatial-holdout \emph{prediction} (Weather2K) and
\emph{reconstruction} of within-image residual structure (Wheat
Head); logistic-normal intervals are in
Appendix~\ref{app:bce-ablation}.

\section{Discussion}\label{sec:discussion}

The Spatial Adapter equips a frozen first-stage predictor with a
structured spatial decomposition of its residual field and a
closed-form estimator of the induced spatial covariance; the added
spatial representation uses fewer than $K(N{+}T)$ parameters, in
addition to a compact residual-trend network.  Covariance estimation and rank determination are tightly coupled:
$K$ is only an optimization-side upper bound, and the
\emph{effective retained rank} is a property of the covariance
estimator, selected data-adaptively by the eigenvalue-thresholding
rule (Eq.~\eqref{eq:wang-rank}, Appendix~\ref{app:cov-correction}).
The Stiefel constraint on $\bm\Phi$ makes these surviving directions
identifiable up to per-column sign; the $(\lambda_{1},\lambda_{2})$
penalties then render them smooth and locally supported, so
thresholding selects among \emph{structured} spatial directions
rather than arbitrary rotations of a low-rank factorization.
\emph{Limitations:} the
covariance~\eqref{eq:cov-est} is unbiased only at the population
basis (plug-in bias
$\|\widehat{\bm\Phi}-\bm\Phi\|_{F}^{2}$ not theoretically
quantified); intervals~\eqref{eq:cgi}--\eqref{eq:lnui} are
asymptotic, with binary-case construction illustrated
diagnostically (Appendix~\ref{app:bce-ablation}, MPIW and
location-level CP under the variance-only surrogate at fixed
$K{=}16$) and regression-side PI calibration on Weather2K deferred
to future work; and $\bm\Omega$ assumes a Euclidean
$\mathcal S$ (no direct extension to graphs or Riemannian
manifolds).  \emph{Future work} (Appendix~\ref{app:extensions}):
multiclass/multilabel responses, simultaneous prediction bands,
state-space forecasting of $\{\bm\alpha_{j}\}$, split-conformal
finite-sample
calibration~\citep{vovk2005algorithmic}, and joint first-stage
fine-tuning.

\begin{ack}
Acknowledgments will be added in the camera-ready version.
\end{ack}

\paragraph{Code and data availability.}
Anonymous repository mirror for double-blind review:
\url{https://anonymous.4open.science/r/spatial-adapter-C2C3/}.

\bibliography{references}


\appendix

\section{Extensions and expanded limitations}
\label{app:extensions}

This appendix expands the compressed limitations and future-work
paragraphs in Section~\ref{sec:discussion}.

\paragraph{Extension to multiclass and multilabel responses.}
Because the Stage~2 objective~\eqref{eq:loss-general} is stated for
a generic link function $g$ and task loss $\ell_{\mathrm{task}}$,
the same ADMM skeleton accommodates multiclass and multilabel
outputs---under the same caveats as the binary case of
Remark~\ref{rem:bce-basis-step}, namely that the $\bm\Phi$- and
$\mathbf Z$-updates become surrogate-based approximations rather
than closed-form solutions of the exact non-Gaussian subproblems.  For an
$M$-class response we replace the scalar latent field
$\bm\eta_{j,i}\in\mathbb R$ with a vector
$\bm\eta_{j,i}\in\mathbb R^{M}$ and use the softmax link $g$
together with the categorical cross-entropy loss; the low-rank
spatial representation is then the per-class field
$\sum_{k=1}^{K}\bm\alpha_{k,j}\,\phi_{k}(\mathbf s_{i})$ with
$\bm\alpha_{k,j}\in\mathbb R^{M}$.  The basis $\bm\Phi$ is shared
across classes so that the parameter count grows only as
$KM(N{+}T)$, and the closed-form covariance estimator of
Section~\ref{subsec:cov-est}---derived under the
Gaussian/identity case---is applied class-wise on the
natural-parameter scale as an approximation, yielding per-class
link-scale intervals analogous to those used in
Section~\ref{subsec:image}.  A sigmoid link with binary
cross-entropy applied entry-wise gives the multilabel variant; in
both cases the $\mathbf Z$-step uses proximal-gradient updates and
the $\bm\Phi$-step reuses the squared-loss spatial-PCA surrogate,
with an exact derivation left to future work.  The
$\tfrac{1}{8}$-quadratic surrogate of
Remark~\ref{rem:bce-basis-step} extends naturally to the
$M$-class softmax case via Böhning's multinomial curvature
bound~\citep{bohning1992multinomial}, which controls the Hessian
of the log-partition uniformly by
$\tfrac{1}{2}(\mathbf I_{M}-\tfrac{1}{M}\mathbf 1\mathbf 1^{\!\top})$;
we do not evaluate this multiclass variant in the present paper
and flag it purely as the analytical justification for the same
surrogate-based basis update.

\paragraph{Simultaneous spatial prediction bands.}
Because the plug-in conditional covariance
$\widehat\Lambda_{\mathrm{cond}}(\mathcal O_{j})$ is available in
closed form, the per-location interval~\eqref{eq:cgi} extends
naturally to a simultaneous band over any finite collection of
query points $\{\mathbf s^{*}_{1},\dots,\mathbf s^{*}_{m}\}$ by
forming the full conditional covariance matrix
$[\widehat{\bm\phi}(\mathbf s^{*}_{a})^{\!\top}\widehat\Lambda_{\mathrm{cond}}\widehat{\bm\phi}(\mathbf s^{*}_{b})+\widehat\sigma^{2}\delta_{ab}]_{a,b}$
and applying a Scheffé- or max-statistic adjustment to the
pointwise quantile.  We evaluate only pointwise intervals in this
paper; simultaneous coverage across regions is a direct by-product
of the covariance structure and is left to future work.

\paragraph{State-space interpretation of the basis scores.}
Because Stage~2 places no prior on $\mathbf A$ across the sample
index $j$, the learned score sequence $\{\bm\alpha_{j}\}_{j=1}^{T}$
can be treated \emph{post hoc} as noisy observations of a
$K$-dimensional latent process and handed to any classical
time-series or state-space model.  In sensor-network tasks where
$j$ indexes time, fitting a first-order vector autoregression
\[
  \bm\alpha_{j}=\bm\Gamma\,\bm\alpha_{j-1}+\mathbf w_{j},
  \qquad \mathbf w_{j}\sim(\mathbf 0,\mathbf Q),
\]
to the learned sequence---by ridge regression or the Yule--Walker
equations---yields one-step-ahead predictions
$\widehat{\bm\alpha}_{T+1}=\widehat{\bm\Gamma}\,\widehat{\bm\alpha}_{T}$
at unseen time points, and hence extrapolated latent fields
$\bm\eta_{T+1,i}=f_{\hat{\bm\theta}}(\mathbf x_{T+1,i})+\mu_{\hat{\bm\psi}}(\mathbf x_{T+1,i})+\sum_{k}\widehat\alpha_{k,T+1}\phi_{k}(\mathbf s_{i})$,
without retraining either the adapter or the first stage.  The cost
is a single $K\!\times\!K$ linear system.

\paragraph{Conformal compatibility.}
The plug-in intervals~\eqref{eq:cgi} are asymptotic; undercoverage
is expected whenever the estimates
$(\widehat{\bm\Phi},\widehat\Lambda,\widehat\sigma^{2})$ are biased
relative to their population counterparts, or when the query
distribution drifts from the training distribution (e.g., from
training stations to held-out stations).  The estimated conditional
variance $\hat v(\mathbf s^{*},j)$ is nonetheless fully compatible
with split conformal prediction~\citep{vovk2005algorithmic}: given a
calibration set disjoint from training, replacing the Gaussian
quantile $z_{\alpha/2}$ with the $(1{-}\alpha)$-quantile of the
normalised conformity scores
$s_{i}=|Y_{i}^{\mathrm{cal}}-\hat\eta_{i}|/\sqrt{\hat v_{i}}$ yields
a conformalised prediction interval with finite-sample marginal
coverage of at least $1-\alpha$ while retaining closed-form
inference.  A full empirical evaluation of this
extension---particularly the choice of calibration set under
spatial distribution drift---is left to future work.

\paragraph{Joint first-stage updates.}
The ADMM formulation also extends to joint first-stage fine-tuning by
reinstating a $\bm\theta$-step alongside the trend update~(T),
turning the post-hoc adapter into a one-stage spatially regularized
training procedure.  We leave a systematic study of when this is
preferable to the two-stage protocol to future work.

\paragraph{Expanded limitations.}
Two further technical items behind the main-text limitations.
\emph{Nonconvex ADMM and convergence:}
our training procedure falls into the class of nonconvex,
stochastic, three-block ADMM~\citep{boyd2011distributed}, for which
global convergence guarantees are generally unavailable; we
establish no rate and rely on a practical early-stopping criterion:
we stop when the dual residual
$\rho\|\mathbf Z^{(k)}-\mathbf Z^{(k-1)}\|_{F}$ and the basis change
$\Delta\bm\Phi$ both fall below a relative tolerance
$\varepsilon^{\text{abs}}+\varepsilon^{\text{rel}}\cdot\max(\|\mathbf Z\|,\|\mathbf Y\|)$,
subject to a minimum number of outer iterations.  Optimization is
stabilised in practice by three choices: (i)~the Stage-1 warm start,
which places the iterates in a favourable region; (ii)~the strongly
convex Gaussian $\mathbf Z$-subproblem
(Proposition~\ref{prop:z-update}), which gives the consensus step a
contraction factor $\rho/(\rho+2)\!\in\!(0,1)$; and
(iii)~infrequent basis updates that are frozen after
$N_{\text{freeze}}$ sweeps, restricting the later stage to the
$(\bm\psi,\mathbf Z,\mathbf U)$ block.  The spatial smoothness
penalty on $\bm\Phi$ additionally restricts the effective search
space.  Empirical primal/dual-residual traces for the synthetic
benchmark (seed~42) are in Appendix~\ref{app:admm-convergence}.
\emph{Rank selection:}
the basis rank $K$ is chosen via the cumulative-variance
rule~\eqref{eq:rank-selection} with a uniform 90\% threshold; the
downstream covariance estimator further applies
\citet{wang2017regularized}'s eigenvalue-thresholding rule, so
weak directions may contribute negligibly or be truncated by the
estimated noise level (with our well-defined reformulation for the
empty-defining-set corner case,
Appendix~\ref{app:cov-correction}).  The threshold
$\tau_{\mathrm{var}}=0.9$ is a heuristic; a full sensitivity
analysis over $K$ or an alternative criterion (e.g.\ information
criterion or scree plot~\citep{cattell1966scree}) is future work.
\emph{Undercoverage sources (detailed):} the four sources behind
limitation~(ii) are (a)~plug-in substitution
$\widehat{\bm\Phi}\to\bm\Phi$ and
$(\widehat\Lambda,\widehat\sigma^{2})\to(\Lambda,\sigma^{2})$
injecting $O(\|\widehat{\bm\Phi}-\bm\Phi\|_{F}^{2})$ bias
into~\eqref{eq:kriging-se}; (b)~deep trend
$f_{\hat{\bm\theta}}+\mu_{\hat{\bm\psi}}$ treated as a fixed point
estimate, so residual fitting error in $\bm\psi$ is not propagated;
(c)~working-independence estimate $\widehat\Lambda$ understating
the marginal second moment of $\bm\alpha_{j}$ when scores are
temporally correlated; and (d)~for binary classification, the
latent-Gaussian calibration on the logit scale inherits the
finite-sample caveats of binomial Wald
intervals~\citep{agresti1998approximate} near the probability-$0/1$
boundaries.

\section{Evaluation metrics: full definitions}
\label{app:metrics}

This appendix collects the formulas and definitions behind the
metric summary in Section~\ref{subsec:metrics}.  Throughout,
$\mathcal I\subset\{1,\dots,T\}\times\{1,\dots,N\}$ is an
evaluation index set and $\widehat{\mathbf Y}=g(\widehat{\bm\eta})$
is the model prediction on the full grid under the fitted latent
field.

\paragraph{Pointwise accuracy.}
In the time-split experiments (Appendix~\ref{subsec:synth}
and Section~\ref{subsec:kaust}), the adapter observes $\mathbf Y$ at all
$N$ locations and reconstructs the spatial field via
$\widehat{\mathbf Y}_{j}=\text{trend}_{j}+\bm\Phi\bm\Phi^{\!\top}(\mathbf Y_{j}-\text{trend}_{j})$;
RMSE therefore measures \emph{reconstruction} quality, not
out-of-sample prediction (which is evaluated separately in
Section~\ref{subsec:weather2k}).  For regression
($g=\mathrm{identity}$) we report
\begin{equation}\label{eq:pointwise_def}
\begin{aligned}
\mathrm{RMSE}_{\mathcal I}
&=\Biggl[\tfrac{1}{|\mathcal I|}\!\!\sum_{(j,i)\in\mathcal I}\!
   \bigl(\mathbf Y_{j,i}-\widehat{\mathbf Y}_{j,i}\bigr)^{2}\Biggr]^{1/2}\!,
&\;
\mathrm{MAE}_{\mathcal I}
&=\tfrac{1}{|\mathcal I|}\!\!\sum_{(j,i)\in\mathcal I}\!
   \bigl|\mathbf Y_{j,i}-\widehat{\mathbf Y}_{j,i}\bigr|,\\[2pt]
R^{2}_{\mathcal I}
&=1-\frac{\sum_{(j,i)\in\mathcal I}\!\bigl(\mathbf Y_{j,i}-\widehat{\mathbf Y}_{j,i}\bigr)^{2}}
        {\sum_{(j,i)\in\mathcal I}\!\bigl(\mathbf Y_{j,i}-\overline{\mathbf Y}_{\mathcal I}\bigr)^{2}},
&\; \overline{\mathbf Y}_{\mathcal I}
&=\tfrac{1}{|\mathcal I|}\!\!\sum_{(j,i)\in\mathcal I}\!\mathbf Y_{j,i}.
\end{aligned}
\end{equation}
For binary classification ($g=\mathrm{sigmoid}$) we report
accuracy, AUC, and $F_{1}$ at a $0.5$ decision threshold
(standard definitions).

\paragraph{Spatial-dependence fidelity.}
We compare predicted and reference spatial covariance matrices
via the relative Frobenius error
\begin{equation}\label{eq:covfrob_def}
\mathrm{CovFrob}
\;=\;
\frac{\bigl\lVert\widehat{\bm\Sigma}_{\text{pred}}
            -\bm\Sigma_{\text{ref}}\bigr\rVert_{F}}
     {\lVert\bm\Sigma_{\text{ref}}\rVert_{F}},
\qquad
\widehat{\bm\Sigma}_{\text{pred}}
=\tfrac{1}{T_{\mathcal I}-1}
 \widetilde{\widehat{\mathbf Y}}^{\,\top}
 \widetilde{\widehat{\mathbf Y}}
 \;\in\;\mathbb R^{N\times N},
\end{equation}
with $T_{\mathcal I}$ the number of repetition indices in
$\mathcal I$ and $\widetilde{\widehat{\mathbf Y}}$ the
repetition-centered prediction matrix.  The reference
$\bm\Sigma_{\text{ref}}$ is the analytic population covariance
when the DGP is known and the empirical covariance otherwise.
As an alternative characterization insensitive to global
rescaling, we also report the semivariogram discrepancy
\begin{equation}\label{eq:semivar-score}
\mathrm{SV}_{\text{score}}
\;=\;
\Biggl[\frac{1}{B}\sum_{b=1}^{B}
  \Bigl(
    \tfrac{\bar{\hat\gamma}_{\text{pred}}(h_{b})-\bar{\hat\gamma}_{\text{obs}}(h_{b})}
          {\bar{\hat\gamma}_{\text{obs}}(h_{b})+\delta}
  \Bigr)^{2}\Biggr]^{1/2},
\end{equation}
where $\bar{\hat\gamma}(\cdot)$ is the Matheron empirical
semivariogram~\citep{mueller2022gstools} averaged over
repetitions, $\{h_{b}\}_{b=1}^{B}$ are adaptive distance bins,
and $\delta>0$ is a denominator stabilizer.  For classification
we probe second-order structure through the closed-form logit
variance of Section~\ref{subsec:cov-est} instead.

\paragraph{Interval calibration.}
For regression (Weather2K spatial holdout) the adapter supplies
$100(1-\alpha)\%$ prediction intervals
$\widehat{\mathrm{PI}}_{j,i}=[\widehat L_{j,i},\widehat U_{j,i}]$,
and we assess calibration through
\begin{equation}\label{eq:cp-mpiw}
\mathrm{CP}
=\tfrac{1}{|\mathcal I|}
 \!\!\sum_{(j,i)\in\mathcal I}\!\!
 \mathbf 1\!\bigl\{\mathbf Y_{j,i}\in\widehat{\mathrm{PI}}_{j,i}\bigr\},
\qquad
\mathrm{MPIW}
=\tfrac{1}{|\mathcal I|}
 \!\!\sum_{(j,i)\in\mathcal I}\!\!
 \bigl(\widehat U_{j,i}-\widehat L_{j,i}\bigr),
\end{equation}
at the nominal level $\alpha=0.05$.  For binary classification
(Wheat Head) the interval~\eqref{eq:lnui} is a logistic-normal
interval on the probability scale, not a prediction interval for
$y\in\{0,1\}$; we therefore report MPIW and a location-level
$\mathrm{CP}=|\mathcal S|^{-1}\sum_{\mathbf s\in\mathcal S}
\mathbf 1\{\bar p(\mathbf s)\in\widehat{\mathrm{PI}}(\mathbf s)\}$
against $\bar p(\mathbf s)=T_{\mathrm{test}}^{-1}\sum_{j}y_{j}(\mathbf s)$,
and the expected calibration error
\begin{equation}\label{eq:ece}
\mathrm{ECE}
\;=\;
\sum_{b=1}^{B}\frac{|B_{b}|}{n}
  \bigl|\bar p_{b}-\bar y_{b}\bigr|,
\end{equation}
where the test predictions are binned into $B$ equal-width bins
of predicted probability, $\bar p_{b}$ and $\bar y_{b}$ are the
mean predicted probability and mean observed frequency in
bin~$b$, and $|B_{b}|$ is the bin count.
$\mathrm{ECE}$ is used only in the Bernoulli basis-update ablation
of Appendix~\ref{app:bce-ablation}, where it discriminates among
surrogate variants on point predictions.

\section{Proposition~\ref*{prop:z-update} and proof}\label{app:proofs}

\begin{proposition}[Closed-form $\mathbf Z$-update under the Gaussian/identity model]
\label{prop:z-update}
Assume $\ell_{\mathrm{data}}(\mathbf M)=\|\mathbf M\|_{F}^{2}$ and
let $\mathbf{Res}_{\mathcal B}\coloneqq\mathbf R_{\mathcal B}(\bm\psi)-\mathbf U_{\mathcal B}$
denote the dually-corrected residual.  The $\mathbf Z$-subproblem
\begin{equation}\label{eq:z-subproblem}
\min_{\mathbf Z}\;\|\mathbf Z\mathbf P^{\!\perp}\|_{F}^{2}
  +\tfrac{\rho}{2}\|\mathbf Z-\mathbf{Res}_{\mathcal B}\|_{F}^{2}
\end{equation}
is $\rho$-strongly convex and its unique global minimizer is
given by~\eqref{eq:z-closed-form}.
\end{proposition}

\begin{proof}
Since $\mathbf P$ and $\mathbf P^{\!\perp}$ are orthogonal
idempotents, any matrix $\mathbf M\in\mathbb R^{L\times N}$
decomposes as
$\mathbf M=\mathbf M\mathbf P+\mathbf M\mathbf P^{\!\perp}$
with $\|\mathbf M\|_{F}^{2}=\|\mathbf M\mathbf P\|_{F}^{2}+\|\mathbf M\mathbf P^{\!\perp}\|_{F}^{2}$.
Applied to the coupling term in~\eqref{eq:z-subproblem},
\[
\|\mathbf Z-\mathbf{Res}_{\mathcal B}\|_{F}^{2}
=\|\mathbf Z\mathbf P-\mathbf{Res}_{\mathcal B}\mathbf P\|_{F}^{2}
+\|\mathbf Z\mathbf P^{\!\perp}-\mathbf{Res}_{\mathcal B}\mathbf P^{\!\perp}\|_{F}^{2},
\]
while the data term $\|\mathbf Z\mathbf P^{\!\perp}\|_{F}^{2}$
depends only on the $\mathbf P^{\!\perp}$-component.  The objective
therefore separates into two independent strongly convex problems.
On the in-basis direction, only the coupling term involves
$\mathbf Z\mathbf P$ and is minimized at
$\mathbf Z\mathbf P=\mathbf{Res}_{\mathcal B}\mathbf P$.  On the
off-basis direction the subproblem in
$\mathbf X\coloneqq\mathbf Z\mathbf P^{\!\perp}$ is
$\|\mathbf X\|_{F}^{2}+(\rho/2)\|\mathbf X-\mathbf{Res}_{\mathcal B}\mathbf P^{\!\perp}\|_{F}^{2}$,
whose first-order condition
$2\mathbf X+\rho(\mathbf X-\mathbf{Res}_{\mathcal B}\mathbf P^{\!\perp})=\mathbf 0$
gives $\mathbf X=(\rho/(\rho+2))\,\mathbf{Res}_{\mathcal B}\mathbf P^{\!\perp}$.
Summing the two components yields~\eqref{eq:z-closed-form}.
Strong convexity follows from the Hessian
$2\mathbf P^{\!\perp}+\rho\mathbf I_{N}$ having minimum eigenvalue
$\rho>0$ on the $\mathbf P$-direction.
\end{proof}

\section{Basis extension, conditional predictive variance, and intervals}
\label{app:conditional-variance}

This appendix supplies the predictive-variance machinery deferred
from Section~\ref{subsec:cov-est}.

\paragraph{Kriging conditional variance.}
Given an observation set $\mathcal O_{j}\subseteq\{1,\dots,N\}$ at
sample $j$, the fixed-rank kriging conditional
variance~\citep{cressie2008fixed} at a query location $\mathbf s^{*}$
is
\begin{equation}\label{eq:kriging-se}
  \hat v(\mathbf s^{*},j)
  \;\coloneqq\;
  \widehat\sigma^{2}
  \;+\;
  \widehat{\bm\phi}(\mathbf s^{*})^{\!\top}\,
  \widehat\Lambda_{\mathrm{cond}}(\mathcal O_{j})\,
  \widehat{\bm\phi}(\mathbf s^{*}),
\end{equation}
with plug-in conditional score covariance
\begin{equation}\label{eq:lambda-cond}
  \widehat\Lambda_{\mathrm{cond}}(\mathcal O_{j})
  \;\coloneqq\;
  \Bigl(\widehat\Lambda^{-1}
        +\widehat\sigma^{-2}\,\bm\Phi_{\mathcal O_{j}}^{\!\top}\bm\Phi_{\mathcal O_{j}}\Bigr)^{-1}.
\end{equation}
When $\mathcal O_{j}=\emptyset$ (new-sample regime),
$\widehat\Lambda_{\mathrm{cond}}=\widehat\Lambda$ and $\hat v$ reduces
to the marginal variance; when $\mathcal O_{j}$ is the
training-station set (Weather2K spatial-holdout,
Section~\ref{subsec:weather2k}),
$\widehat\Lambda_{\mathrm{cond}}$ shrinks strictly below
$\widehat\Lambda$ and recovers the standard kriging variance.

\begin{proposition}[Plug-in conditional Gaussian predictive law]\label{prop:gaussian-predictive}
Under the optional Gaussian working model
$\bm\alpha_{j}\sim\mathcal N(\mathbf 0,\Lambda)$,
$\bm\epsilon_{j}\sim\mathcal N(\mathbf 0,\sigma^{2}\mathbf I)$,
plugging the ADMM estimates
$(\widehat{\bm\Phi},\widehat\Lambda,\widehat\sigma^{2},\widehat{\bm\theta},\widehat{\bm\psi})$
into the linear-Gaussian conditioning identities gives
\begin{equation}\label{eq:cond-gauss}
  \eta(\mathbf s^{*},j)\mid\mathbf r_{j,\mathcal O_{j}}
  \;\stackrel{\text{approx}}{\sim}\;
  \mathcal N\!\bigl(\hat\eta(\mathbf s^{*},j),\,\hat v(\mathbf s^{*},j)\bigr),
\end{equation}
with conditional mean
\begin{equation}\label{eq:point-pred}
  \hat\eta(\mathbf s^{*},j)
  =
  f_{\hat{\bm\theta}}(\mathbf x_{j}^{*})+\mu_{\hat{\bm\psi}}(\mathbf x_{j}^{*})
  + \widehat{\bm\phi}(\mathbf s^{*})^{\!\top}\,\widehat{\bm\alpha}_{j}(\mathcal O_{j}),
\quad
  \widehat{\bm\alpha}_{j}(\mathcal O_{j})=\widehat\sigma^{-2}\widehat\Lambda_{\mathrm{cond}}\bm\Phi_{\mathcal O_{j}}^{\!\top}\mathbf r_{j,\mathcal O_{j}},
\end{equation}
($\widehat{\bm\alpha}_{j}(\emptyset)=\mathbf 0$) and conditional
variance $\hat v$ from~\eqref{eq:kriging-se}--\eqref{eq:lambda-cond}.
This is a direct application of Gaussian conditioning with plug-in
moments; we state it as a convenient summary, not a finite-sample
coverage guarantee.
\end{proposition}

\paragraph{Prediction intervals.}
Proposition~\ref{prop:gaussian-predictive} gives the
$100(1-\alpha)\%$ plug-in regression interval for the latent
natural parameter,
\begin{equation}\label{eq:cgi}
  \hat\eta(\mathbf s^{*},j)\;\pm\;z_{\alpha/2}\,\sqrt{\hat v(\mathbf s^{*},j)}.
\end{equation}
For a binary task ($g=\sigma$), the probability
$p(\mathbf s^{*},j)=\sigma(\eta(\mathbf s^{*},j))$ is the
\emph{pushforward} of the plug-in interval through $\sigma$:
\begin{equation}\label{eq:lnui}
  \bigl[\,\sigma\bigl(\hat\eta(\mathbf s^{*},j)-z_{\alpha/2}\sqrt{\hat v(\mathbf s^{*},j)}\bigr),\;
         \sigma\bigl(\hat\eta(\mathbf s^{*},j)+z_{\alpha/2}\sqrt{\hat v(\mathbf s^{*},j)}\bigr)\,\bigr].
\end{equation}
Equation~\eqref{eq:lnui} quantifies uncertainty about the
\emph{predicted probability induced by the latent uncertainty},
not a coverage statement for the Bernoulli outcome
$Y\in\{0,1\}$; because $\sigma$ is monotone the endpoints map
directly without a delta-method step.  The link-scale residual used
to form $\hat\eta$ and $\hat v$ is
$\mathbf r_{j}=g^{\dagger}(\mathbf y_{\varepsilon,j})-f_{\hat{\bm\theta}}(\mathbf x_{j})-\mu_{\hat{\bm\psi}}(\mathbf x_{j})$
with $\varepsilon$-clamped labels ($\varepsilon=10^{-7}$) so the
logit stays finite.

Per-sample cost is $O(|\mathcal O_{j}|K^{2}+K^{3})$; per-query cost
is $O(K^{2})$.

\paragraph{Basis extension to unseen locations.}
The estimated basis $\widehat{\bm\Phi}\in\mathbb R^{N\times K}$
is defined only on the $N$ observed locations
$\{\mathbf s_{1},\dots,\mathbf s_{N}\}$ used during Stage~2.  For a
query location $\mathbf s^{*}$ that is \emph{not} among these
rows---e.g., a held-out weather station or an evaluation-time image
patch---the basis row
$\widehat{\bm\phi}(\mathbf s^{*})\in\mathbb R^{K}$ is obtained by
passing each learned column of $\widehat{\bm\Phi}$ through a
deterministic thin-plate-spline smoother of $\mathcal S$ and
evaluating the smoother at $\mathbf s^{*}$.  All downstream
conditional predictive formulas are therefore conditional on this
interpolated basis row: out-of-sample uncertainty at $\mathbf s^{*}$
is driven by both the estimated covariance structure
$(\widehat\Lambda,\widehat\sigma^{2})$ and a separately interpolated
basis representation $\widehat{\bm\phi}(\mathbf s^{*})$, and the
quality of the latter interacts with the roughness penalty
$\bm\Omega$ in~\eqref{eq:loss-general}.  When $\mathbf s^{*}$
happens to coincide with one of the observed locations the
interpolation step is trivial and
$\widehat{\bm\phi}(\mathbf s^{*})$ is read directly from the
corresponding row of $\widehat{\bm\Phi}$.  We emphasize that this
out-of-sample basis extension is a deterministic \emph{post-hoc}
interpolation step and is \emph{not} jointly optimized within the
Stage~2 objective~\eqref{eq:loss-general}.

\paragraph{Conditional predictive variance: working-model context.}
Formulas \eqref{eq:kriging-se}--\eqref{eq:lambda-cond} have the
structure of the fixed-rank kriging predictive
variance~\citep{cressie2008fixed}, with the true
$(\bm\Phi,\Lambda,\sigma^{2})$ replaced by the ADMM estimates
$(\widehat{\bm\Phi},\widehat\Lambda,\widehat\sigma^{2})$; they are
therefore a \emph{plug-in} approximation rather than a true BLUP.
Stage~2 training itself does not require a Gaussian prior on the
scores $\bm\alpha_{j}$: the basis, the score matrix, and the
covariance estimator~\eqref{eq:cov-est} are derived purely from
the first two moments.  The downstream conditional prediction and
uncertainty formulas do invoke an \emph{optional Gaussian working
model} $\bm\alpha_{j}\sim\mathcal N(\mathbf 0,\Lambda)$,
$\bm\epsilon_{j}\sim\mathcal N(\mathbf 0,\sigma^{2}\mathbf I)$ to
turn the estimated second moments into closed-form conditional
means and variances; under that assumption
$\widehat\Lambda_{\mathrm{cond}}(\mathcal O_{j})$ coincides with
the Bayesian posterior covariance of $\bm\alpha_{j}$.

\paragraph{Cost of~\eqref{eq:kriging-se}--\eqref{eq:lambda-cond}.}
For each sample $j$ we form and factor (e.g., via Cholesky) the
$K\!\times\!K$ matrix
$\widehat\Lambda_{\mathrm{cond}}(\mathcal O_{j})^{-1}=\widehat\Lambda^{-1}+\widehat\sigma^{-2}\bm\Phi_{\mathcal O_{j}}^{\!\top}\bm\Phi_{\mathcal O_{j}}$
at one-time cost $O(|\mathcal O_{j}|K^{2}+K^{3})$; each subsequent
query location $\mathbf s^{*}$ then costs only $O(K^{2})$ for the
bilinear form in~\eqref{eq:kriging-se}.  In the sparse-observation
regime $|\mathcal O_{j}|\ll K$, the inverse can be rewritten via
the Sherman--Morrison--Woodbury identity as
$\widehat\Lambda_{\mathrm{cond}}(\mathcal O_{j})=\widehat\Lambda-\widehat\Lambda\bm\Phi_{\mathcal O_{j}}^{\!\top}\bigl(\widehat\sigma^{2}\mathbf I_{|\mathcal O_{j}|}+\bm\Phi_{\mathcal O_{j}}\widehat\Lambda\bm\Phi_{\mathcal O_{j}}^{\!\top}\bigr)^{-1}\bm\Phi_{\mathcal O_{j}}\widehat\Lambda$,
reducing the computation to an $|\mathcal O_{j}|\!\times\!|\mathcal O_{j}|$
system at cost $O(|\mathcal O_{j}|^{3}+|\mathcal O_{j}|K^{2})$.
In either form the cost is independent of the full grid size~$N$.

\paragraph{Two score notions.}
The ADMM-recovered score matrix
$\widehat{\mathbf A}=\widehat{\mathbf Z}\widehat{\bm\Phi}$ from
Section~\ref{subsec:admm} is the \emph{in-sample} factor score
associated with the fitted consensus on the training indices,
whereas $\widehat{\bm\alpha}_{j}(\mathcal O_{j})$
in~\eqref{eq:point-pred} is the \emph{plug-in conditional} score
used for out-of-sample prediction under the Gaussian working
model; the two coincide in expectation when $\mathcal O_{j}$ is
the full training index set and the ADMM consensus has converged.

\section{ADMM step derivations}\label{app:admm-steps}

This appendix supplies the full derivations of the four ADMM
subproblems summarised in Section~\ref{subsec:admm}.

\paragraph{(T) Trend step.}
Only the coupling term of the augmented Lagrangian depends on
$\bm\psi$.  With $(\bm\Phi,\mathbf Z_{\mathcal B},\mathbf U_{\mathcal B})$
held fixed, we take a single SGD step on $\bm\psi$ to decrease
$\tfrac{\rho}{2}\|\mathbf Z_{\mathcal B}-\mathbf R_{\mathcal B}(\bm\psi)+\mathbf U_{\mathcal B}\|_{F}^{2}$;
the gradient flows through $\mathbf M_{\bm\psi}$ while the frozen
first stage $f_{\hat{\bm\theta}}$ does not contribute.  Intuitively,
this step pushes the residual trend $\mu_{\bm\psi}$ to match the
dually-corrected consensus $\mathbf Z_{\mathcal B}+\mathbf U_{\mathcal B}$.

\paragraph{(B) Basis step.}
With $(\bm\psi,\mathbf Z_{\mathcal B},\mathbf U_{\mathcal B})$
fixed, the $\bm\Phi$-subproblem
\[
  \min_{\bm\Phi}\;
  \ell_{\mathrm{data}}(\mathbf Z_{\mathcal B}\mathbf P^{\!\perp})
  +\lambda_{1}\operatorname{tr}\!\bigl(\bm\Phi^{\!\top}\bm\Omega\bm\Phi\bigr)
  +\lambda_{2}\|\bm\Phi\|_{1}
  \quad\text{s.t.}\quad
  \bm\Phi^{\!\top}\bm\Phi=\mathbf I_{K}
\]
has the same form as the regularized spatial PCA
of~\citet{wang2017regularized} applied to the mean-centered
consensus $\mathbf Z_{\mathcal B,c}=\mathbf Z_{\mathcal B}-\overline{\mathbf Z}_{\mathcal B}$.
For the Gaussian case $\ell_{\mathrm{data}}(\mathbf M)=\|\mathbf M\|_{F}^{2}$,
using $\|\mathbf Z\mathbf P^{\!\perp}\|_{F}^{2}
 =\|\mathbf Z\|_{F}^{2}-\operatorname{tr}(\bm\Phi^{\!\top}\mathbf Z_{c}^{\!\top}\mathbf Z_{c}\bm\Phi)$,
the $\lambda_{2}=0$ subproblem is solved in closed form by taking
$\bm\Phi$ to be the top-$K$ eigenvectors of the symmetrized
penalized Gram matrix
$\tfrac{1}{2}(\mathbf C_{\mathcal B}-\lambda_{1}\bm\Omega+(\mathbf C_{\mathcal B}-\lambda_{1}\bm\Omega)^{\!\top})$,
where $\mathbf C_{\mathcal B}=\mathbf Z_{\mathcal B,c}^{\!\top}\mathbf Z_{\mathcal B,c}$,
followed by SVD-based projection onto the Stiefel manifold
$\{\bm\Phi:\bm\Phi^{\!\top}\bm\Phi=\mathbf I_{K}\}$.  When
$\lambda_{2}>0$ enforces column sparsity, we replace the direct
eigen-step by an iteratively reweighted-$\ell_{1}$
(IRL$_{1}$)~\citep{candes2008enhancing} inner loop: starting from
the eigen-warm-start, we alternate a gradient step
$\bm\Phi\!+\!\alpha(\mathbf C_{\mathcal B}\!-\!\lambda_{1}\bm\Omega)\bm\Phi$,
an entry-wise soft-threshold with reweighted level
$\alpha\lambda_{2}/(|\bm\Phi|+\varepsilon)$, and an SVD retraction
onto the Stiefel manifold, until successive iterates of
$\bm\Phi$ agree to a tolerance (Algorithm~\ref{alg:admm}).  After a
fixed number of ADMM sweeps $N_{\text{freeze}}$ the basis is
frozen and only the trend $\bm\psi$, consensus $\mathbf Z$, and
dual $\mathbf U$ continue to update; we found this stabilisation
necessary to avoid small late-stage oscillations in $\bm\Phi$ that
do not affect prediction accuracy.

\paragraph{(Z) Consensus step.}
The $\mathbf Z_{\mathcal B}$-subproblem
\begin{equation}\label{eq:z-step}
  \min_{\mathbf Z\in\mathbb R^{L\times N}}\;
  \ell_{\mathrm{data}}(\mathbf Z\mathbf P^{\!\perp})
  +\tfrac{\rho}{2}\bigl\|\mathbf Z-\mathbf R_{\mathcal B}(\bm\psi)+\mathbf U_{\mathcal B}\bigr\|_{F}^{2}
\end{equation}
admits a closed-form solution in the Gaussian/identity case
(Proposition~\ref{prop:z-update}).  For the Bernoulli case we take
one proximal-gradient step~\citep{Parikh2014}
\[
  \mathbf Z_{\mathcal B}
  \;\leftarrow\;
  \mathbf Z_{\mathcal B}
  -\tau\!\Bigl[
     \nabla_{\mathbf Z}\,
       \mathrm{BCE}\!\bigl(\sigma(\mathbf Z_{\mathcal B}\mathbf P^{\!\perp}),\mathbf Y_{\mathcal B}\bigr)
     +\rho\bigl(\mathbf Z_{\mathcal B}-\mathbf R_{\mathcal B}(\bm\psi)+\mathbf U_{\mathcal B}\bigr)
  \Bigr],
\]
with a fixed step size $\tau$; in our implementation the
coupling gradient is averaged per entry of the mini-batch (the
effective penalty is $\rho/(LN)$) for numerical stability, a
standard rescaling convention in stochastic
ADMM~\citep{boyd2011distributed}.  One or two such iterations
suffice in all our classification experiments.

\paragraph{(U) Dual update.}
$\mathbf U_{\mathcal B}\leftarrow
   \mathbf U_{\mathcal B}+\mathbf Z_{\mathcal B}-\mathbf R_{\mathcal B}(\bm\psi)$.

\section{Per-iteration cost and memory of the dense ADMM implementation}
\label{app:periter-cost}

The trend step~(T) is a standard mini-batch update on the residual
network $\mu_{\bm\psi}$; the consensus step~(Z) scales as $O(LNK)$
via the rank-$K$ factorization of $\mathbf P^{\!\perp}$; and the
dual step~(U) is an elementwise update at $O(LN)$.
The basis step~(B) explicitly forms the $N\!\times\!N$ Gram
$\mathbf C=\mathbf Z_{c}^{\!\top}\mathbf Z_{c}$ at $O(LN^{2})$
time and $O(N^{2})$ memory, then calls a dense symmetric
eigensolver (\texttt{torch.linalg.eigh}) at $O(N^{3})$ and
truncates to the top $K$ eigenvectors; when $\lambda_{2}>0$,
a short IRL$_{1}$ inner loop adds a few $O(NK^{2})$
matrix--vector products.

In our experiments $N$ is moderate---$N=256$ for the
$16\!\times\!16$ Wheat Head patches, $\lesssim 2000$ stations
for Weather2K, and a few hundred locations on the KAUST grid
and synthetic configurations---and the basis step is executed
infrequently due to a configurable cadence $N_{\text{phi}}$
and is skipped entirely after a warm-start horizon
$N_{\text{freeze}}$, so the $O(N^{2})$ Gram and $O(N^{3})$
eigensolve never become the training bottleneck.  One to three
ADMM sweeps per mini-batch suffice for convergence in all
reported runs.

For larger $N$, the basis step can in principle be replaced by
an implicit-operator formulation
$\mathbf M\mapsto\mathbf Z_{c}^{\!\top}(\mathbf Z_{c}\mathbf M)-\lambda_{1}\bm\Omega\mathbf M$
combined with Lanczos- or randomized-SVD-type eigensolvers,
avoiding explicit $N\!\times\!N$ materialization; we do not
evaluate this variant here and flag it as a scalable path for
future work.

\section{Rank-selection corner case in the closed-form covariance estimator}
\label{app:cov-correction}

This appendix supplies the technical detail referenced from the
footnote in Section~\ref{subsec:cov-est}.  We recall the
rank-selection rule of \citet{wang2017regularized}, identify a
corner case in which it is not well defined, and give an equivalent
reformulation that is well defined for all inputs.

\paragraph{The original rule.}
Let $\widehat d_{1}\ge\cdots\ge\widehat d_{K}$ denote the
eigenvalues of $\widehat{\bm\Phi}^{\!\top}\mathbf S\,\widehat{\bm\Phi}$
and let $\tau\ge 0$ be the smoothness penalty.  For any
$L\in\{1,\dots,K\}$ define the noise estimate under $L$ active
components,
\begin{equation}\label{eq:sigma-L}
\widehat\sigma_{L}^{2}
\;\coloneqq\;
\frac{1}{N-L}\Bigl(\operatorname{tr}\mathbf S
       \;-\;\sum_{k=1}^{L}\!\bigl(\widehat d_{k}-\tau\bigr)\Bigr).
\end{equation}
Equation~(12) of \citet{wang2017regularized} sets the estimated
rank to
\begin{equation}\label{eq:wang-rank}
\widehat L_{\,\mathrm{W}}
\;=\;
\max\!\Bigl\{\,L\in\{1,\dots,K\}\;:\;
              \widehat d_{L}-\tau \;>\;\widehat\sigma_{L}^{2}\,\Bigr\},
\end{equation}
and then plugs $\widehat L_{\,\mathrm{W}}$ into the piecewise
$\widehat\sigma^{2}$ formula of their Eq.~(11).
$\widehat L_{\,\mathrm{W}}$ is the \emph{effective retained rank}
of the covariance model: once the Stage-2 basis budget $K$ is
fixed via~\eqref{eq:rank-selection}, \eqref{eq:wang-rank}
decides data-adaptively how many of the $K$ directions actually
enter $\widehat\Sigma_{\mathbf r}$.

\paragraph{The corner case.}
The set on the right-hand side of~\eqref{eq:wang-rank} can be
\emph{empty}: if every leading eigenvalue
$\widehat d_{L}$ fails to exceed $\tau+\widehat\sigma_{L}^{2}$
(typically in weak-signal regimes where the deep trend has already
absorbed most of the structure, e.g.\ the $10/10/80$ KAUST split
of Section~\ref{subsec:kaust} or low-prevalence patches in the
GWHD experiment of Section~\ref{subsec:image}), then
$\max\emptyset$ is undefined and so is
$\widehat L_{\,\mathrm{W}}$, leaving the downstream estimator
$\widehat\Sigma_{\mathbf r}$ undefined.

\paragraph{Equivalent well-defined reformulation.}
The estimated rank should reflect the number of components whose
shrunk eigenvalue is strictly positive, and the empty-set case
should correspond to $\widehat L=0$.  We therefore define
\begin{equation}\label{eq:rank-fixed}
\widehat L
\;\coloneqq\;
\#\bigl\{k\in\{1,\dots,K\}\;:\;\widehat\lambda_{k}>0\bigr\},
\qquad
\widehat\lambda_{k}
\;=\;
\max\!\bigl(\widehat d_{k}-\widehat\sigma^{2}-\tau,\;0\bigr),
\end{equation}
where $\widehat\sigma^{2}$ is selected jointly with $\widehat L$ by
the piecewise rule
\begin{equation}\label{eq:sigma-fixed}
\widehat\sigma^{2}
\;=\;
\begin{cases}
\widehat\sigma_{\widehat L}^{2}, & \widehat L\ge 1,\\[2pt]
\frac{1}{N}\operatorname{tr}\mathbf S, & \widehat L = 0,
\end{cases}
\end{equation}
and $\widehat\sigma_{L}^{2}$ is given by~\eqref{eq:sigma-L}.  On the
non-empty branch~\eqref{eq:rank-fixed} returns the same value as
\eqref{eq:wang-rank}; on the empty branch it returns
$\widehat L=0$, in which case~\eqref{eq:sigma-fixed} reduces to the
$\widehat d_{1}\le\tau$ case of Eq.~(11) of
\citet{wang2017regularized} and the plug-in estimator collapses to
the isotropic
$\widehat\Sigma_{\mathbf r}=\widehat\sigma^{2}\mathbf I_{N}$.  Thus
\eqref{eq:rank-fixed}--\eqref{eq:sigma-fixed} extend the original
rule continuously to all admissible inputs without changing it on
the regime in which it was originally defined.

\paragraph{Relation to rank selection in our experiments.}
In all reported experiments, the basis budget $K$ is selected first
from the cumulative explained variance of the Stage-1 residual
sample covariance, as summarized in
Appendix~\ref{app:rank-selection}.  This $K$ is not the final
covariance rank; it is an operational upper bound for the adapter
basis.  The effective retained rank is then determined inside the
covariance estimator by the Wang--Huang eigenvalue-thresholding
rule.  Thus the reported low-rank covariance model is obtained in
two steps: a scree-plot budget for the candidate spatial basis,
followed by spectral thresholding that keeps only the directions
with positive shrunk eigenvalues.  The
reformulation~\eqref{eq:rank-fixed}--\eqref{eq:sigma-fixed} does
not change this procedure on the non-empty branch; it only makes
the degenerate empty-set branch well defined by returning
$\widehat L=0$, in which case the estimator reduces to the
isotropic residual-noise model
$\widehat\Sigma_{\mathbf r}=\widehat\sigma^{2}\mathbf I_{N}$.


\section{ADMM convergence diagnostics}
\label{app:admm-convergence}

Figure~\ref{fig:admm-convergence} plots the primal and dual
ADMM residuals on the synthetic 1D benchmark of
Appendix~\ref{subsec:synth} for both the unregularized and
regularized adapter.

\begin{figure}[ht]
\centering
\includegraphics[width=\linewidth]{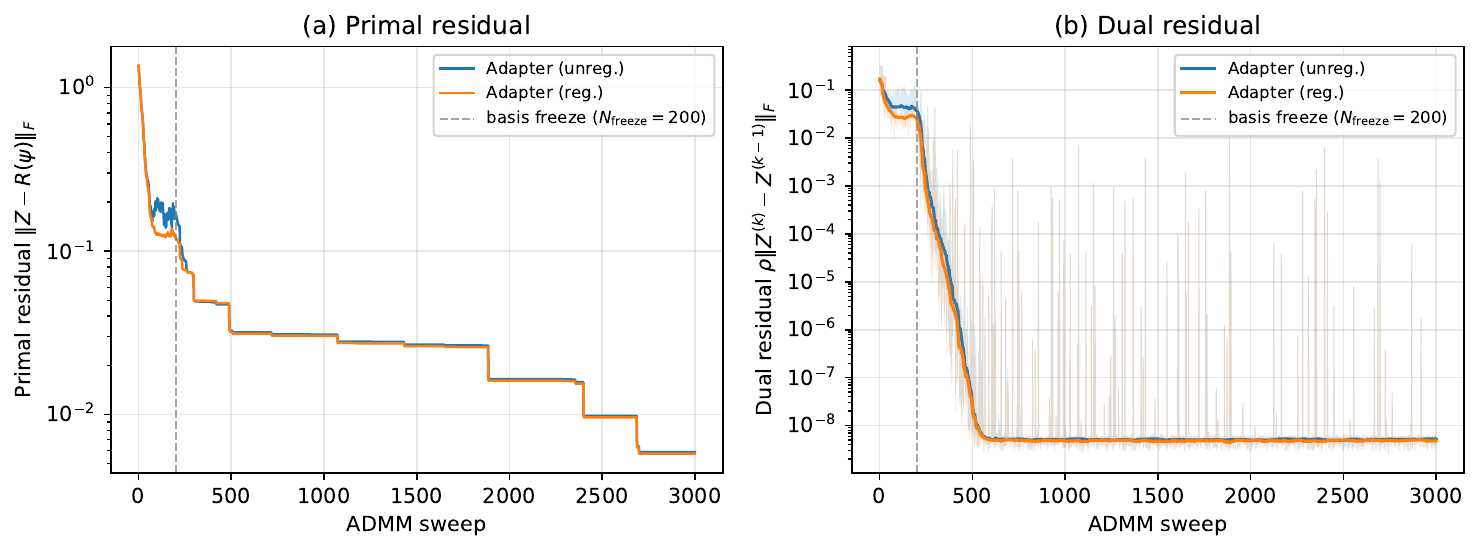}
\caption{\textbf{ADMM convergence on the synthetic benchmark
(seed~42).}
(\emph{a})~The primal residual
$\|\mathbf Z-\mathbf R(\bm\psi)\|_{F}$ decays monotonically in
both settings, with staircase steps corresponding to basis updates
before the freeze horizon $N_{\mathrm{freeze}}=200$.
(\emph{b})~The dual residual
$\rho\|\mathbf Z^{(k)}-\mathbf Z^{(k-1)}\|_{F}$ (rolling median
shown as solid line; raw trace as shading).  The unregularized
variant converges to ${\sim}10^{-8}$; the regularized variant shows
intermittent spikes typical of stochastic ADMM with a non-smooth
$\ell_{1}$ penalty, while its rolling median decays steadily.}
\label{fig:admm-convergence}
\end{figure}


\section{Synthetic experiment: full data-generating process}
\label{app:synth-dgp}
\label{subsec:synth}

This appendix supplies the full setup and data-generating process
behind the synthetic benchmark of Section~\ref{subsec:synth}.

\paragraph{Setup.}
The first-stage predictor $f_{\hat{\bm\theta}}$ is a frozen
ordinary-least-squares (OLS) linear fit of $\mathbf x_{i,t}$ on
$Y(s_i,t)$, solved in closed form on the training split; no
gradient-based training is performed in Stage~1.  Stage~2 runs the
mini-batch ADMM of Algorithm~\ref{alg:admm} with batch size
$L{=}64$ over the time axis; because Stage~1 is frozen, the (T)
trend step is a no-op and only the (B)/(Z)/(U) steps update per
iteration.  Means and standard errors are across $30$ random seeds.
We use the 1D toy deliberately to keep the Stage-2 mechanism
inspectable; the data-generating process---not the geometry---is
where we aim to reflect conditions a first-stage predictor would
face in practice.

\paragraph{Climate motivation.}
The Gaussian profile $\phi(s)=\exp(-s^{2})/\|\exp(-s^{2})\|_{2}$
mimics the Indian Ocean Basin (IOB) mode---a broad, sign-definite
monopole that climate reanalyses identify as the leading mode of
tropical Indian Ocean SST variability, driven physically by
ENSO-induced surface heat-flux anomalies communicated through an
atmospheric bridge, and exhibiting seasonal-scale
persistence~\citep{klein1999bridge, yang2007iob, xie2009capacitor}.
A persistent AR(1) score $\alpha_t$ ($\alpha_\rho{=}0.8$)
reproduces the basin mode's seasonal decorrelation.  Read as a
climate-reanalysis skeleton, $\bm\beta^{\!\top}\mathbf x_{i,t}$
plays the role of a regression estimate from ancillary
meteorological predictors and $\alpha_t\phi(s_i)$ is the residual
basin-scale anomaly mode that the regression cannot
explain---precisely the component the adapter is designed to
recover.

\paragraph{Data-generating process.}
\begin{enumerate}[label=\textbf{(\roman*)},leftmargin=1.5em,itemsep=2pt]
  \item \textbf{Time-varying covariates.}
        For $t\in[0,2\pi]$ generate the smooth base series
        \[
          \small
          \begin{aligned}
            &\text{temp}(t)=15+10\sin t+\varepsilon_{\text{temp}}(t),\quad
            &&\varepsilon_{\text{temp}}\!\sim\!\mathcal N(0,0.1^{2}), \\
            &\text{wind}(t)=|\,\mathcal N(3,0.1^{2})\,|,\\[-1pt]
            &\text{hum}(t)=70-0.5\,\text{temp}(t)+\varepsilon_{\text{hum}}(t),\quad
            &&\varepsilon_{\text{hum}}\!\sim\!\mathcal N(0,0.1^{2}), \\
            &\text{pres}(t)=1000+20\,\varphi_{\mathcal N(\pi,1)}(t),\\[-1pt]
            &\text{poll}(t)=50+2\,\text{temp}(t)-3\,\text{wind}(t)+\varepsilon_{\text{poll}}(t),\quad
            &&\varepsilon_{\text{poll}}\!\sim\!\mathcal N(0,0.1^{2}),
          \end{aligned}
        \]
        replicate them to every location with i.i.d.\ $\mathcal N(0,0.01^{2})$
        perturbations and stack as
        $\mathbf x_{i,t}\!\in\!\mathbb R^{5}$.

  \item \textbf{Optional drifts.}
        \[
          \gamma_t
            =0.6\,\gamma_{t-1}+\varepsilon_{\gamma,t},\;
            \varepsilon_{\gamma,t}\!\sim\!\mathcal N(0,1-0.6^{2}),
          \qquad
          \alpha_t
            =\alpha_\rho\,\alpha_{t-1}
             +\varepsilon_{\alpha,t},\;
            \varepsilon_{\alpha,t}\!\sim\!\mathcal N\bigl(0,5^{2}(1-0.8^{2})\bigr).
        \]
        Unless stated otherwise we disable the global drift
        ($\gamma_t\!\equiv\!0$) and retain only the spatial score
        $\alpha_t$, so that the $N$ locations share a single rank-1
        spatial component rather than having its effect dominated by a
        constant-in-space offset.

  \item \textbf{Observation model.}
        \[
          Y(s_i,t)
          =50+\mathbf x_{i,t}^{\!\top}\boldsymbol\beta+\gamma_t+\alpha_t\,\phi(s_i)+\varepsilon_{i,t},
          \quad
          \varepsilon_{i,t}\!\sim\!\mathcal N(0,4^{2}),
        \]
\end{enumerate}
where $\boldsymbol\beta=(10,-9.5,2.6,-1.3,1.2)^{\!\top}$.

\paragraph{Results.}

\begin{table}[htbp]
  \centering
  \caption{Reconstruction performance on the 1D synthetic benchmark.
  RMSE, MAE, and $\mathrm{Cov}_{\mathrm{Frob}}$ shown in units of
  $10^{-2}$; $R^{2}$ in percent.  The adapter observes $\mathbf Y$
  at all locations, so these metrics audit reconstruction quality
  and covariance recovery against the known ground truth.  Mean
  (std.\ error) across 30 random seeds.}
  \label{tab:temporal_results}
  \begin{tabular}{@{}lcccc@{}}
    \toprule
    \textbf{Model}
      & \textbf{RMSE} ($\times10^{-2}$)
      & \textbf{MAE} ($\times10^{-2}$)
      & $\mathbf{R^{2}}\;(\%)$
      & $\mathbf{Cov_{\mathrm{Frob}}}$ ($\times10^{-2}$) \\
    \midrule
    OLS (backbone only)    & 5.723\,(0.003) & 4.566\,(0.003) & 97.810\,(0.004) & 1.752\,(0.012) \\
    + Adapter (unreg.)     & 5.717\,(0.003) & 4.561\,(0.003) & 97.814\,(0.004) & 1.747\,(0.011) \\
    + Adapter (reg.)       & \textbf{5.711}\,(0.003) & \textbf{4.557}\,(0.003) & \textbf{97.819}\,(0.004) & \textbf{1.726}\,(0.009) \\
    \bottomrule
  \end{tabular}
\end{table}


\section{Rank selection: per-experiment cumulative-variance curves}
\label{app:rank-selection}

This appendix supplies the data-driven rank selection behind the
$K$ values stated in Section~\ref{subsec:twostage}.  For each
dataset we plot the cumulative explained variance of the Stage-1
residual sample covariance and mark the smallest $K$ satisfying
$\sum_{i=1}^{K}\hat\lambda_{i}/\sum_{i=1}^{N}\hat\lambda_{i}\ge\tau_{\mathrm{var}}=0.9$.
The resulting $K$ values are used directly in our experiments.

\begin{table}[htbp]
  \centering
  \caption{Rank-selection summary.  $K$ is deterministic per
           split and comes from the cumulative-variance rule at
           $\tau_{\mathrm{var}}=0.9$ (full training split).  The
           \emph{effective rank} is the integer-valued output of
           the Wang--Huang eigenvalue-thresholding rule
           (Eq.~\eqref{eq:wang-rank}, Appendix~\ref{app:cov-correction})
           used downstream by $\widehat\Sigma_{\mathbf r}$
           (Section~\ref{subsec:cov-est}); we report its mean over
           the 30 Stage-2 seeds, which is why the column carries a
           decimal part even though every individual seed yields an
           integer count of retained directions.}
  \label{tab:rank-selection}
  \begin{tabular}{@{}llccc@{}}
    \toprule
    \textbf{Dataset} & \textbf{Backbone / split}
      & $\boldsymbol{N}$
      & $\boldsymbol{K}$
      & \textbf{effective rank} \\
    \midrule
    KAUST             & STDK, 10/10/80        & 1000 & 46  & 43.2 \\
    Weather2K         & STDK, 80/10/10        & 187  & 14  & 6.5 \\
    Weather2K         & STDK, 10/10/80        & 187  & 14  & 6.5 \\
    Weather2K holdout & STDK, 20\% space      & 150  & 13  & 6.2 \\
    Wheat Head        & ResNet-152            & 256  & 153 & 128.0 \\
    Wheat Head        & ConvNeXt-T            & 256  & 155 & 132.8 \\
    Wheat Head        & ViT-B/16              & 256  & 155 & 133.7 \\
    Wheat Head        & SAM (ViT-H)           & 256  & 155 & 131.0 \\
    \bottomrule
  \end{tabular}
\end{table}

\begin{figure}[htbp]
  \centering
  \begin{subfigure}{\linewidth}
    \includegraphics[width=\linewidth]{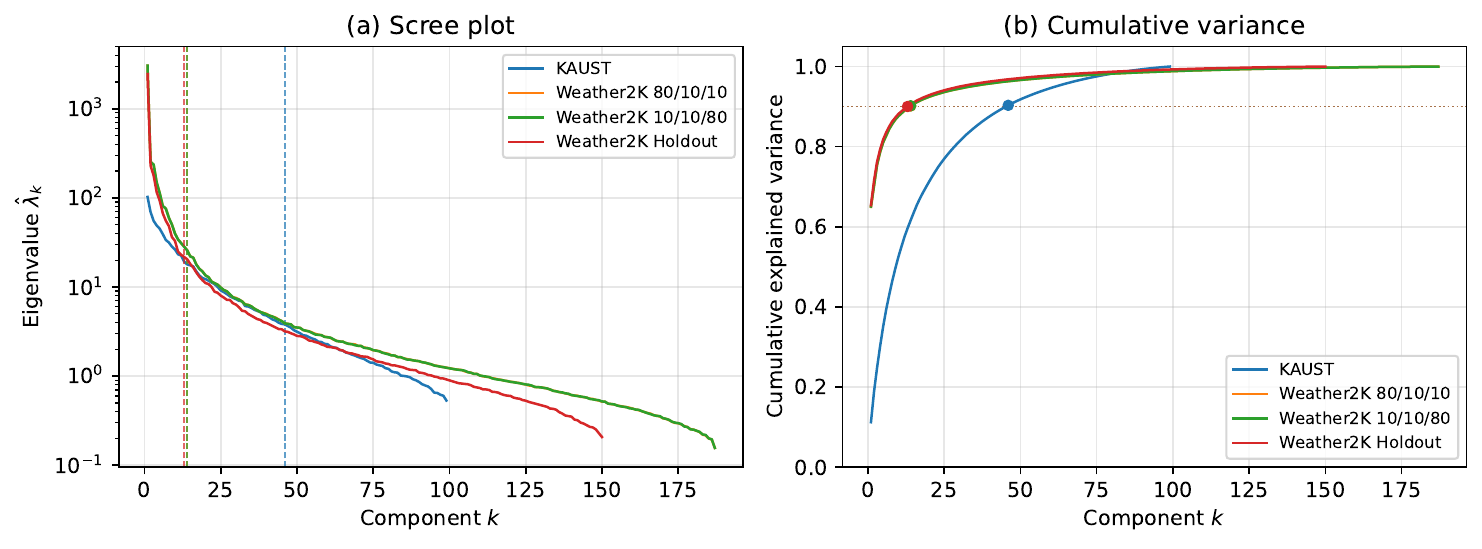}
    \caption{KAUST and Weather2K (three splits): cumulative explained
             variance of the Stage-1 residual sample covariance.}
    \label{fig:rank-sel-other}
  \end{subfigure}\\[0.4em]
  \begin{subfigure}{\linewidth}
    \includegraphics[width=\linewidth]{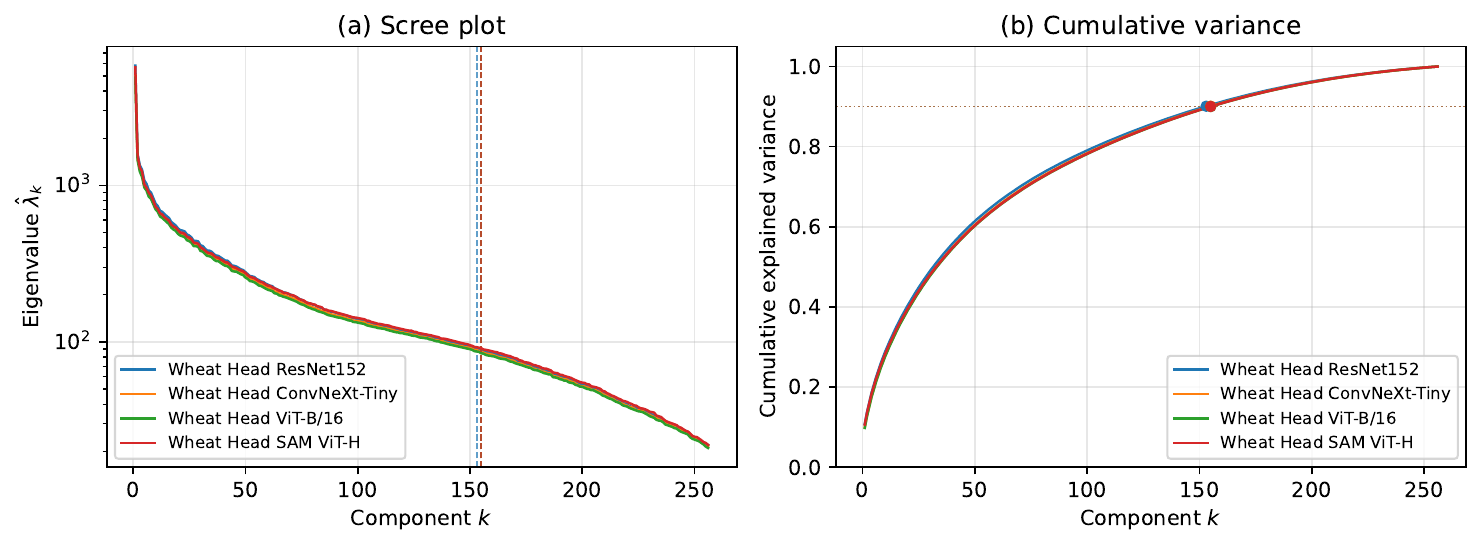}
    \caption{Wheat Head (4 backbones): cumulative explained variance
             on the $16\!\times\!16$ patch grid.  The much slower
             decay relative to the regression datasets reflects the
             near-isotropic residual spectrum induced by frozen
             generic-domain backbones.}
    \label{fig:rank-sel-wheat}
  \end{subfigure}
  \caption{Per-experiment cumulative-variance curves for the
           rank-selection rule of
           Section~\ref{subsec:twostage}.  Vertical lines mark the
           chosen $K$ at $\tau_{\mathrm{var}}=0.9$.}
  \label{fig:rank-selection}
\end{figure}

\section{KAUST setup details}
\label{app:kaust-setup}

The backbone $f_{\hat{\bm\theta}}$ is an STDK
model~\citep{lin2023enhancements} on $(x,y,t)$ coordinates only
(no external covariates), trained on the $10\%$ train window with
Adam (lr $10^{-3}$, weight decay $10^{-4}$, batch $512$, up to
$350$ epochs, patience $30$ on validation RMSE).  The
``under-trained'' regime arises from the narrow time window rather
than a reduced epoch budget.  The $80\%$ test share maximises
replicates for CovFrob and SV$_{\mathrm{score}}$ (which remain
informative even when $T_{\mathrm{test}}\!<\!N$ because both
$\widehat{\bm\Sigma}_{\mathrm{pred}}$ and
$\widehat{\bm\Sigma}_{\mathrm{obs}}$ are computed from the same
$T_{\mathrm{test}}$ samples).  Case \texttt{2b} of the competition
provides multiple independent replicates from the same stationary
Gaussian process; we use \texttt{2b\_8} arbitrarily, and the
Monte Carlo varies only the spatial subsample and initialisations
(not the underlying field realisation), so
Table~\ref{tab:ablation_val} is a relative comparison of
validation criteria on the same field.

\section{Wheat Head setup details}
\label{app:wheat-setup}

This appendix supplies the framework mapping, backbone details,
and two-stage procedure behind Section~\ref{subsec:image}.

\paragraph{Mapping to the adapter framework.}
The within-image patch grid maps directly onto the spatial GLM of
Section~\ref{sec:methodology}:
\begin{itemize}[leftmargin=1.5em,itemsep=2pt]
  \item \textbf{Spatial index}
        $\mathbf s=(h,w)\in\{1,\dots,16\}^{2}$:
        the patch position on the normalized image grid
        (analogous to station coordinates in Weather2K).
  \item \textbf{Sample index} $j=1,\dots,T$:
        each image acts as one repetition
        ($T=$ number of images).
  \item \textbf{Response} $Y(j,\mathbf s)\in\{0,1\}$:
        presence/absence of a wheat head in patch~$\mathbf s$
        of image~$j$.
  \item \textbf{Backbone} $f_{\hat{\bm\theta}}(\mathbf s)$:
        per-patch logit from a frozen classifier
        (each patch classified independently).
  \item \textbf{Spatial basis} $\bm\Phi(\mathbf s)$:
        learned smooth, orthonormal basis functions over the
        $16\!\times\!16$ grid.  Each $\phi_k$ captures a
        \emph{positional pattern} in wheat-head occurrence stable
        across images---e.g.\ centre-heavy vs.\ edge-heavy
        distributions from camera framing or field geometry.
  \item \textbf{Spatial covariance}
        $\widehat\Sigma
        =\widehat{\bm\Phi}\,\widehat\Lambda\,\widehat{\bm\Phi}^{\!\top}
        +\widehat\sigma^{2}\mathbf I$:
        quantifies the clustering that if patch $(3,3)$ has a wheat
        head, nearby patches $(3,4)$ and $(4,3)$ are likely to as well.
  \item \textbf{Plug-in prediction interval}: the per-patch
        logit variance from $\widehat\Sigma$ is pushed through the
        sigmoid to produce logistic-normal intervals
        for the class probability.
\end{itemize}

\paragraph{Backbones.}
Four architectures spanning classical CNNs, modern CNNs,
Transformers, and foundation models:
\begin{enumerate}[leftmargin=1.8em,itemsep=1pt]
  \item \textbf{ResNet-152}~\citep{he2016deep} --- deep CNN baseline
        (no built-in spatial awareness);
  \item \textbf{ConvNeXt-T}~\citep{liu2022convnext} --- modern CNN
        (no explicit spatial bias);
  \item \textbf{ViT-B/16}~\citep{dosovitskiy2021vit} --- Vision
        Transformer (learned positional embedding only);
  \item \textbf{SAM (ViT-H)}~\citep{kirillov2023sam} --- Segment
        Anything Model with strong built-in spatial reasoning; we
        freeze the image encoder and attach a linear per-patch head.
\end{enumerate}
All four backbones are frozen at their pre-trained weights
(ImageNet for the first three, the original SAM checkpoint for
ViT-H); only a lightweight classification head is trained on GWHD
patch features in Stage~1.

\paragraph{Two-stage procedure.}
In Stage~1, each frozen pre-trained backbone extracts a feature
vector per patch and a lightweight classification head is trained
on those features.  Unlike the regression experiments where the
first-stage predictor is trained on the target data, here the
backbone has never seen the GWHD domain; the adapter can be viewed
as a spatially structured adaptation layer on a frozen model.  In
Stage~2, both backbone and head are frozen and the adapter fits
$(\bm\psi,\bm\Phi,\mathbf Z)$ with $g=\sigma$ and BCE loss, using
the proximal-gradient $\mathbf Z$-step of Section~\ref{subsec:admm};
evaluation is reconstruction (test-image patch labels observed,
residual logits projected onto the learned basis), consistent with
the protocol of Section~\ref{subsec:kaust}.

\paragraph{Parameter efficiency.}
Table~\ref{tab:wheat-param-efficiency} reports, per backbone, the
frozen first-stage parameter count (pretrained vision backbone plus
the Stage-1 classification head, both held fixed during Stage~2)
against the Stage-2 parameters introduced by the adapter itself:
the residual trend correction $\mu_{\text{net}}$ (a two-layer MLP)
and the learned spatial basis
$\widehat{\bm\Phi}\in\mathbb R^{N\times K}$ with $N\!=\!256$ and
$K\!\in\!\{153,155\}$.  Across all four backbones the adapter adds
between $0.009\%$ (SAM ViT-H) and $0.32\%$ (ConvNeXt-T) of the
frozen first-stage parameter count.

\begin{table}[htbp]
  \centering
  \caption{GWHD parameter-efficiency summary.  \emph{Frozen first stage}
           collects the pretrained vision backbone (feature extractor,
           classifier head removed) and the Stage-1 classification
           head; both are held fixed during Stage~2.
           \emph{Stage-2 added} is the adapter's residual trend
           correction $\mu_{\text{net}}$ plus the learned spatial
           basis $\widehat{\bm\Phi}\in\mathbb R^{N\times K}$.
           Counts produced by
           \texttt{count\_params.py} in the released code.}
  \label{tab:wheat-param-efficiency}
  \begin{tabular}{@{}lrrrr@{}}
    \toprule
    \textbf{Backbone} & \textbf{Frozen total}
      & \textbf{Stage-2 added} & \textbf{Ratio}
      & \textbf{Breakdown (backbone / head \textbar{} $\mu_{\text{net}}$ / $\Phi$)} \\
    \midrule
    ResNet-152             & 58.70\,M  & 170.37\,K & 0.290\,\% & 58.14\,M / 557.57\,K \textbar{} 131.20\,K / 39.17\,K \\
    ConvNeXt-T             & 28.05\,M  &  88.96\,K & 0.317\,\% & 27.82\,M / 229.89\,K \textbar{}  49.28\,K / 39.68\,K \\
    ViT-B/16               & 86.80\,M  &  88.96\,K & 0.102\,\% & 86.57\,M / 229.89\,K \textbar{}  49.28\,K / 39.68\,K \\
    SAM (ViT-H)$^{\dagger}$ & 637.10\,M &  56.19\,K & 0.009\,\% & 637.00\,M / 98.82\,K \textbar{}  16.51\,K / 39.68\,K \\
    \bottomrule
  \end{tabular}

  \vspace{0.25em}
  {\footnotesize $^{\dagger}$ SAM image-encoder count taken from the
   Segment Anything model card~\citep{kirillov2023sam}; the other
   three backbones are counted from the torchvision feature
   extractors actually used in Stage~1 (classifier head removed).}
\end{table}

\section{Weather2K held-out: extended figures and full table}
\label{app:weather2k-held-out-extra}

Supplementary evidence for the held-out spatial prediction story
of Section~\ref{subsec:weather2k}.
Table~\ref{tab:weather2k_holdout} gives the full multi-seed
numbers; Fig.~\ref{fig:weather2k_map} shows the station split and
Fig.~\ref{fig:weather2k_correction} visualises the signed adapter
correction surface.  The main-body kriging snapshot is
Fig.~\ref{fig:weather2k_holdout_main}; the main-body per-station
RMSE figure is Fig.~\ref{fig:weather2k_rmse}.

\begin{table}[htbp]
  \centering
  \caption{Spatial prediction at $20\%$ held-out Weather2K stations.
           The conditional-kriging predictor uses the adapter's
           basis at training stations, interpolated via thin-plate
           splines to each held-out location and conditioned on the
           training-station residuals at each test time.  Mean
           (std.\ error) across 30 seeds.}
  \label{tab:weather2k_holdout}
  \begin{tabular}{@{}lc@{}}
    \toprule
    \textbf{Model} & \textbf{RMSE} \\
    \midrule
    STDK (no adapter)        & 7.667\,(0.034) \\
    Adapter (unreg.)         & 2.705\,(0.017) \\
    Adapter (reg.)           & \textbf{2.664}\,(0.015) \\
    \bottomrule
  \end{tabular}
\end{table}

\begin{figure}[ht]
  \centering
  \includegraphics[width=0.6\linewidth]{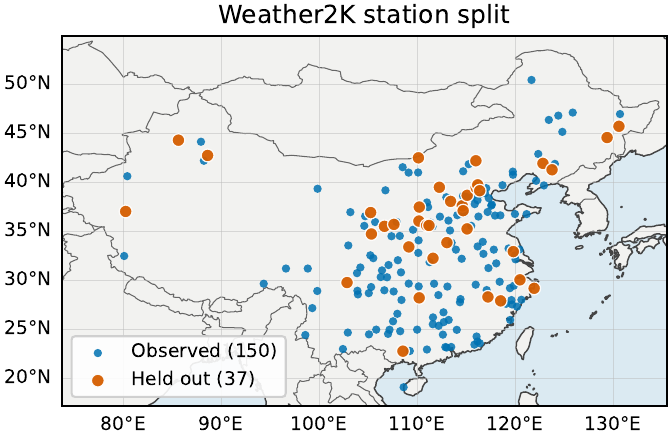}
  \caption{Station split (observed vs.\ held-out).}
  \label{fig:weather2k_map}
\end{figure}

\begin{figure}[ht]
  \centering
  \includegraphics[width=0.85\linewidth]{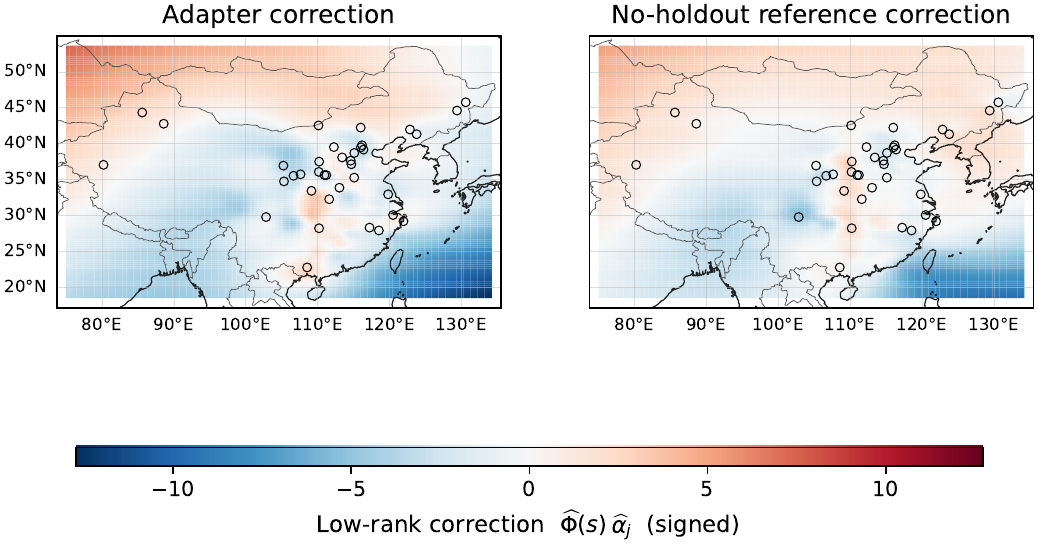}
  \caption{Signed low-rank correction
           $\widehat{\bm\Phi}(\mathbf s)^{\!\top}\widehat{\bm\alpha}_{j}$
           at the same test time as Fig.~\ref{fig:weather2k_holdout_main},
           on a shared diverging scale centred at zero.
           \textbf{Left:} adapter trained only on the $80\%$ observed
           stations (TPS-interpolated to held-out locations).
           \textbf{Right:} no-holdout reference trained with the
           $20\%$ held-out stations included, so the basis is read
           off in-sample and the correction is therefore what the
           observed-only adapter is trying to recover.  Held-out
           stations are outlined in black.  Regions where the two
           differ mark where observed-only training over- or
           under-shoots the all-stations reference.}
  \label{fig:weather2k_correction}
\end{figure}

\section{Weather2K data-split ablation}
\label{app:weather2k-split}

Complementary to the held-out spatial prediction in
Section~\ref{subsec:weather2k}, this appendix reports a time-split
ablation with \emph{all stations observed}.  We run two
contrasting train/val/test allocations, $80/10/10$ and $10/10/80$,
to examine how the split affects both prediction accuracy and
spatial covariance recovery: varying the training proportion
reveals whether the adapter compensates for a weak backbone,
while varying the test proportion controls the reliability of
covariance evaluation (cf.\ Section~\ref{subsec:metrics}).

\begin{table}[htbp]
  \centering
  \caption{Weather2K reconstruction results under different time
           splits (all stations observed).
           Mean (std.\ error) across 30 seeds.}
  \label{tab:weather2k}
  \resizebox{\linewidth}{!}{%
  \begin{tabular}{@{}llcccc@{}}
    \toprule
    \textbf{Split} & \textbf{Model}
      & \textbf{RMSE}
      & \textbf{MAE}
      & $\mathbf{R^{2}}\;(\%)$
      & $\mathbf{Cov_{\mathrm{Frob}}}$ \\
    \midrule
    \multirow{3}{*}{80/10/10}
      & STDK             & 4.547\,(0.011)  & 3.487\,(0.009)  & $-$64.04\,(1.23)   & 0.749\,(0.002) \\
      & Adapter (unreg.) & 2.026\,(0.019)  & 1.550\,(0.015)  & 63.58\,(0.97)      & 0.276\,(0.004) \\
      & Adapter (reg.)   & \textbf{1.899}\,(0.011)  & \textbf{1.450}\,(0.009)  & \textbf{67.99}\,(0.63)  & \textbf{0.248}\,(0.003) \\
    \midrule
    \multirow{3}{*}{10/10/80}
      & STDK             & 7.667\,(0.034)  & 6.239\,(0.031)  & $-$211.75\,(2.17)  & 0.992\,(0.000) \\
      & Adapter (unreg.) & 2.679\,(0.017)  & 2.056\,(0.013)  & 56.71\,(0.54)      & 0.223\,(0.002) \\
      & Adapter (reg.)   & \textbf{2.634}\,(0.016)  & \textbf{2.022}\,(0.012)  & \textbf{58.11}\,(0.51)  & \textbf{0.220}\,(0.002) \\
    \bottomrule
  \end{tabular}%
  }
\end{table}

\section{Regularization path and covariance evolution on the synthetic benchmark}
\label{app:reg-path}

\begin{figure}[ht]
    \centering
    \includegraphics[width=\linewidth]{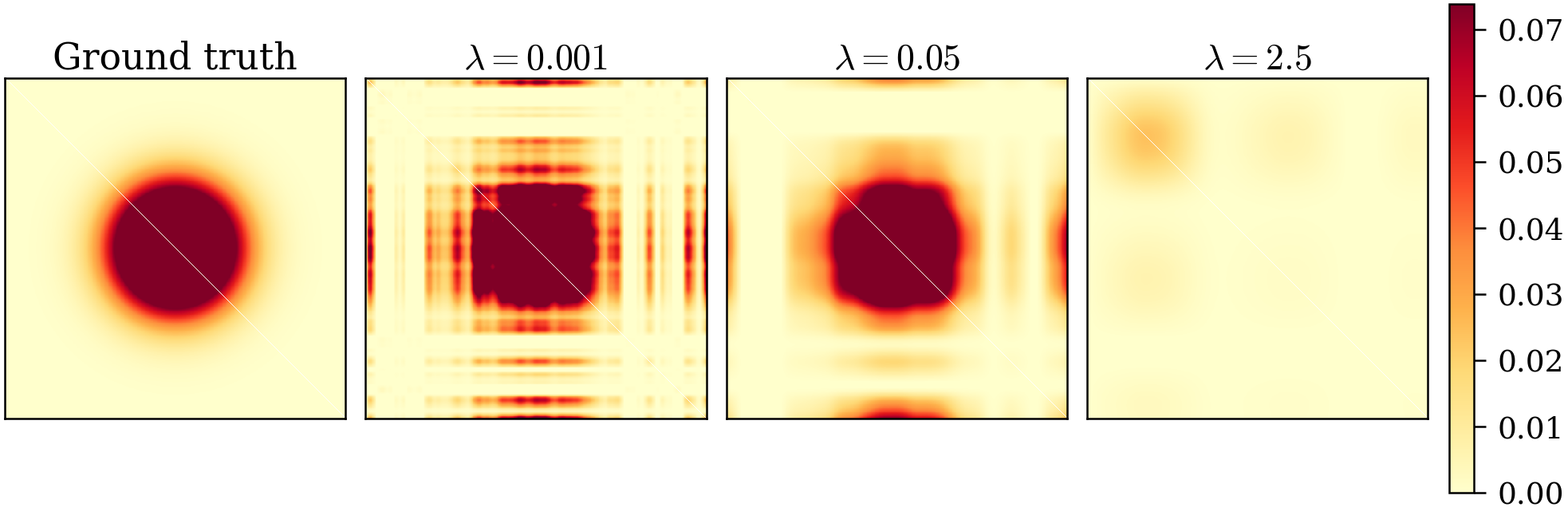}
    \caption{\textbf{Estimated spatial covariance
      $\widehat\Sigma_{\mathbf r}$ along the regularization path}
      on the synthetic benchmark of Appendix~\ref{subsec:synth}
      (median-representative seed, shared colour scale).
      Three regimes are visible: under-regularization injects
      spurious off-diagonal structure; $\lambda^{\star}$ recovers
      the ground-truth rank-1 covariance; over-regularization
      collapses toward the noise floor.  The accompanying
      basis-alignment / CovFrob path summary is in main-body
      Figure~\ref{fig:synth} (right); the four-panel main-text
      comparison at $\Sigma_{\mathrm{true}}$ / OLS /
      unreg.\ / reg.\ is Figure~\ref{fig:synth} (left).  Corresponding
      learned bases are in Figure~\ref{fig:spatial_basis_evo}.}
    \label{fig:spatial_reg_effect}
\end{figure}

\section{Learned basis across the regularization path}
\label{app:basis-evolution}

Figure~\ref{fig:spatial_basis_evo} shows the learned basis
$\hat{\bm\phi}(s)$ at three representative $\lambda$ values
from the path of Figure~\ref{fig:spatial_reg_effect}, overlaid on
the ground-truth rank-1 Gaussian kernel
$\bm\phi(s)=e^{-s^{2}}/\|e^{-s^{2}}\|_2$.  Under-regularization
($\lambda\!=\!0.013$) leaves high-frequency ripple; the validation
optimum ($\lambda^{\star}\!=\!0.13$) recovers $\bm\phi$ almost
exactly; over-regularization ($\lambda\!=\!6.3$) collapses the
basis onto an off-centre direction, consistent with the
$|\langle\hat{\bm\phi},\bm\phi\rangle|$ trace in main-body
Figure~\ref{fig:synth} (right).

\begin{figure}[ht]
\centering
\includegraphics[width=\linewidth]{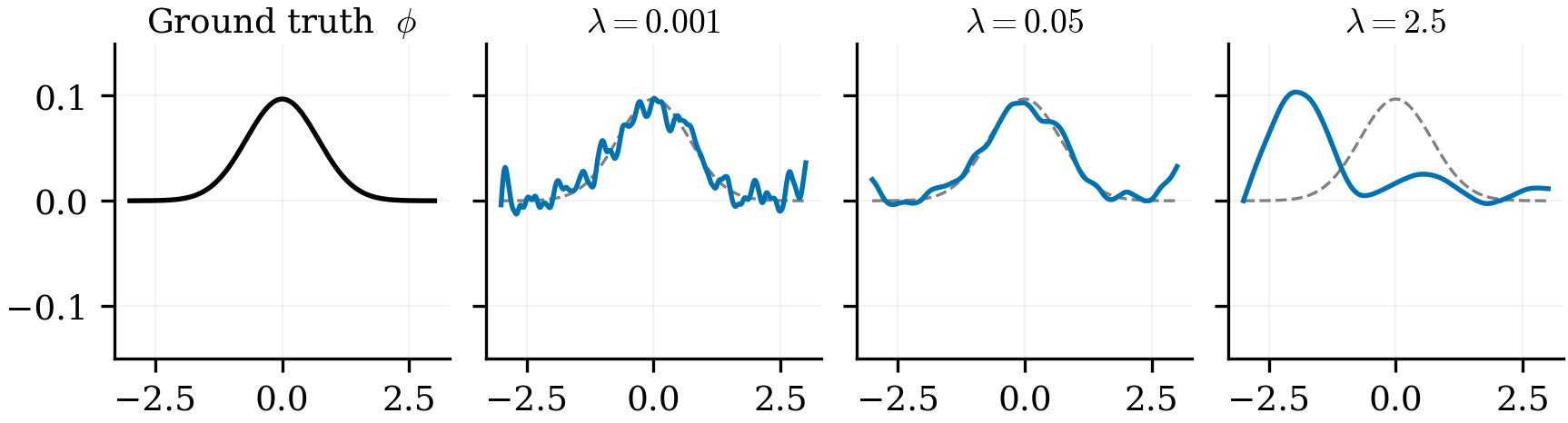}
\caption{\textbf{Learned basis $\hat{\bm\phi}(s)$ across the
regularization path (seed~42).}  Dashed grey: ground truth
$\bm\phi(s)$; solid blue: learned $\hat{\bm\phi}(s)$ after sign
alignment to $\bm\phi$.  Companion to
Figure~\ref{fig:spatial_reg_effect}.}
\label{fig:spatial_basis_evo}
\end{figure}


\section{Experimental configurations}\label{app:configs}

Table~\ref{tab:experiment-configs} summarizes the key
hyperparameters for each experiment.  All STDK first-stage
hyperparameters follow the official
implementation.\footnote{\url{https://github.com/pratiknag/Space-Time.DeepKriging/blob/main/50kSimulation-space-time_DeepKriging.ipynb}}
The Wheat Head backbones (ResNet-152, ConvNeXt-T, ViT-B/16,
SAM~ViT-H) use their original ImageNet / SAM pre-trained weights
without fine-tuning the encoder; only a lightweight classification
head is trained in Stage~1.

\begin{table}[ht]
\centering
\caption{Experimental configurations.  ``First stage'' describes
the frozen predictor; ``Adapter'' describes the Stage~2 ADMM
settings.  Seeds are contiguous starting from the listed start
value.}
\label{tab:experiment-configs}
\small
\begin{tabular}{@{}lcccc@{}}
\toprule
& \textbf{Synthetic 1D}
& \textbf{KAUST}
& \textbf{Weather2K}
& \textbf{Wheat Head} \\
\midrule
\multicolumn{5}{@{}l}{\textit{Data}} \\
$N$ (locations)
  & 512 & 1\,000 & 150 (20\% held out in space) & 256 \\
$T$ (samples)
  & 1\,024 & 1\,000 & 1\,000 & ${\sim}3\,400$ images \\
Train / Val / Test
  & 70/15/15 & 10/10/80 & 10/10/80 (time) & 70/15/15 \\
Task
  & regression & regression & regression & classification \\
\midrule
\multicolumn{5}{@{}l}{\textit{First stage}} \\
Model
  & OLS & STDK & STDK & frozen backbone \\
Optimizer
  & closed-form & Adam & Adam & AdamW \\
LR
  & --- & $10^{-3}$ & $10^{-3}$ & $10^{-3}$ \\
Epochs (max)
  & --- & 350 & 350 & 10 \\
Early stopping
  & --- & patience 30 & patience 30 & fixed epochs \\
Batch size
  & --- & 512 & 512 & 32 \\
\midrule
\multicolumn{5}{@{}l}{\textit{Adapter (Stage~2 ADMM)}} \\
$K$ (basis rank)
  & 1 & 46 & 13 & 153--155 \\
$\rho$ (ADMM penalty)
  & 1.0 & 5.0 & 5.0 & 5.0 \\
Max ADMM iters
  & 3\,000 & 3\,000 & 3\,000 & 200 \\
Min outer iters
  & 20 & 100 & 100 & 50 \\
Convergence tol
  & $10^{-4}$ & $10^{-4}$ & $10^{-4}$ & $10^{-4}$ \\
$\phi$ update cadence
  & 5 & 5 & 5 & 5 \\
$\phi$ freeze after
  & 200 & 100 & 100 & 100 \\
Trend LR ($\mu_{\bm\psi}$)
  & $10^{-2}$ & $10^{-3}$ & $10^{-3}$ & $10^{-3}$ \\
Batch size
  & 64 & 128 & 128 & 32 \\
$(\tau_1,\tau_2)$ range
  & $[10^{-4},10^{8}]$ & $[10^{-4},10^{8}]$ & $[10^{-4},10^{8}]$
  & $[10^{-4},10^{2}]$ \\
Optuna trials
  & 20 & 50 & 50 & 30 \\
\midrule
\multicolumn{5}{@{}l}{\textit{Replication}} \\
Seed start
  & 1 & 1 & 1 & 42 \\
$n$ replications
  & 30 & 30 & 30 & 30 \\
\bottomrule
\end{tabular}
\end{table}

\paragraph{Compute.}
All experiments were run on a container instance with a single
NVIDIA Tesla~V100 SXM2 32~GB GPU, 4 CPU cores (Intel Xeon~Gold
class), and 90~GB of host memory ($\approx$84~GiB usable); feature
extraction and Stage-2 training were run sequentially on the same
instance.  Stage-1 fitting (OLS, STDK, or frozen-backbone feature
extraction with a linear head) is the dominant cost for KAUST and
Weather2K; Stage-2 ADMM is lightweight at $N\!\le\!1000$ and
$K\!\le\!155$.  Median per-seed wall times are
$\approx\!13$~s for Synthetic, $\approx\!2$--$3$~min for KAUST and
two of the three Weather2K settings, and $\approx\!6$--$14$~min for
the GWHD adapter sweeps; the outlier is the Weather2K $80/10/10$
split at $\approx\!1.2$~h per seed because the longer training
window increases both Stage-1 STDK fitting and the $50$-trial
Optuna sweep.  A full breakdown by experiment --- Stage~1,
unregularized adapter, Optuna regularized sweep, and per-seed total
--- is given in Table~\ref{tab:compute-walltime}.

\begin{table}[htbp]
  \centering
  \caption{Median wall-clock cost per component across seeds
    (NVIDIA Tesla V100 SXM2 32~GB).  \emph{Stage~1:} first-stage
    fit (STDK for KAUST and Weather2K; identity for Synthetic;
    frozen backbone with a linear head for GWHD, see note below).
    \emph{Unreg:} a single adapter ADMM solve at
    $\tau_1\!=\!\tau_2\!=\!0$.
    \emph{Reg/trial:} one Optuna trial (adapter ADMM solve at a
    sampled $(\tau_1,\tau_2)$).
    \emph{Reg total:} the full Optuna sweep
    ($n_{\text{trials}}$ trials).
    \emph{Per-seed total:} Stage~1 $+$ Reg total, i.e.\ the cost of
    reproducing one seed's final result.}
  \label{tab:compute-walltime}
  \small
  \begin{tabular}{@{}lrrrrrrr@{}}
    \toprule
    \textbf{Experiment} & $n_{\text{seeds}}$ & $n_{\text{trials}}$
      & \textbf{Stage 1} & \textbf{Unreg} & \textbf{Reg/trial}
      & \textbf{Reg total} & \textbf{Per-seed total} \\
    \midrule
    Synthetic (1D)             & 100 & 20 & ---     & 0.7\,s  & 0.6\,s  & 12.8\,s & 12.8\,s \\
    KAUST$^{\dagger}$          & 100 & 50 & 9.5\,s  & 0.4\,s  & 2.8\,s  & 2.4\,min & 2.5\,min \\
    Weather2K held-out         &  30 & 50 & 1.6\,min & 0.4\,s  & 0.7\,s  & 33.1\,s & 2.1\,min \\
    Weather2K $10/10/80$       & 100 & 50 & 1.9\,min & 0.4\,s  & 0.6\,s  & 29.7\,s & 2.4\,min \\
    Weather2K $80/10/10$       &  30 & 50 & 6.7\,min & 1.3\,min & 1.3\,min & 1.09\,h & 1.21\,h \\
    GWHD SAM ViT-H$^{\ddagger}$      & 30 & 30 & ---     & 11.8\,s & 13.3\,s &  6.6\,min &  6.6\,min \\
    GWHD ConvNeXt-T$^{\ddagger}$     & 30 & 30 & ---     & 15.9\,s & 16.6\,s &  8.3\,min &  8.3\,min \\
    GWHD ViT-B/16$^{\ddagger}$       & 30 & 30 & ---     & 15.8\,s & 16.5\,s &  8.2\,min &  8.2\,min \\
    GWHD ResNet-152$^{\ddagger}$     & 30 & 30 & ---     & 29.2\,s & 28.2\,s & 14.1\,min & 14.1\,min \\
    \bottomrule
  \end{tabular}

  \vspace{0.25em}
  {\footnotesize
   $^{\dagger}$ A single row is reported for KAUST because the three
   tuning targets (RMSE, CovFrob, SV-score) share the same Stage-1
   pipeline and adapter cost; only the Optuna objective differs.
   $^{\ddagger}$ GWHD rows exclude the one-off backbone
   feature-extraction cost, which is cached across seeds within a
   backbone and does not enter per-seed total.  The Weather2K
   $80/10/10$ row is the cost outlier because its $80\%$ training
   window increases both STDK Stage-1 fitting and the $50$-trial
   adapter sweep.}
\end{table}


\section{Bernoulli basis-update ablation}
\label{app:bce-ablation}

This appendix derives the remainder bound stated in
Remark~\ref{rem:bce-basis-step} and supplies the empirical audit
of the variance-only heuristic.

\paragraph{Taylor expansion and remainder bound.}
The $\tfrac{1}{8}$-quadratic surrogate follows from the uniform
curvature bound
$\sup_{x}\sigma(x)(1-\sigma(x))=\tfrac{1}{4}$.
Applying the expansion
$\log(1+e^{x})=\log 2+\tfrac{1}{2}x+\tfrac{1}{8}x^{2}+O(x^{4})$
entrywise to $(\mathbf Z\mathbf P^{\!\perp})_{j,i}$ yields
\begin{equation}\label{eq:bce-taylor}
\mathrm{BCE}\bigl(\sigma(\mathbf Z\mathbf P^{\!\perp}),\mathbf Y\bigr)
=\mathrm{const}(\mathbf Y)
+\bigl\langle\tfrac{1}{2}\mathbf 1-\mathbf Y,\,\mathbf Z\mathbf P^{\!\perp}\bigr\rangle_{F}
+\tfrac{1}{8}\|\mathbf Z\mathbf P^{\!\perp}\|_{F}^{2}
+R(\mathbf Z\mathbf P^{\!\perp}).
\end{equation}
Writing
$\mathbf X=\mathbf Z\mathbf P^{\!\perp}$, a pointwise
fourth-order Taylor remainder bound for the log-sigmoid applied
entrywise and summed via
$\sum_{j,i}x_{j,i}^{4}\le\|\mathbf X\|_{\infty}^{2}\|\mathbf X\|_{F}^{2}$
gives
\begin{equation}\label{eq:bce-remainder}
  \bigl|R(\mathbf X)\bigr|
  \;\le\;
  C\,\|\mathbf X\|_{\infty}^{2}\,\|\mathbf X\|_{F}^{2}
\end{equation}
for some constant $C>0$ on any bounded neighborhood of the origin.
The relative error
$|R(\mathbf X)|/(\tfrac{1}{8}\|\mathbf X\|_{F}^{2})\lesssim\|\mathbf X\|_{\infty}^{2}$
does not vanish with sample size alone but rather with the
per-entry off-basis magnitude---confirming that the surrogate is
accurate when the learned basis captures most of the residual signal.

Reducing~\eqref{eq:bce-taylor} to a quadratic form in $\bm\Phi$
via $\mathbf P^{\!\perp}=\mathbf I_{N}-\bm\Phi\bm\Phi^{\!\top}$
and the cyclic trace identity yields a label-dependent target matrix
$\mathbf C_{\mathrm{BCE}}(\mathbf Z,\mathbf Y)=\operatorname{sym}\!\bigl((\tfrac{1}{2}\mathbf 1-\mathbf Y)^{\!\top}\mathbf Z\bigr)+\tfrac{1}{8}\mathbf Z^{\!\top}\mathbf Z$
whose top-$K$ eigenvectors would give an exact second-order update.
Our default heuristic drops the label-dependent
$\operatorname{sym}(\cdots)$ correction and uses only the
label-free variance component
$\mathbf C_{\mathcal B}=\mathbf Z_{\mathcal B,c}^{\!\top}\mathbf Z_{\mathcal B,c}$.
A principled alternative is a single IRLS step with label-adjusted
working responses~\citep{wood2017generalized}.

\paragraph{Empirical audit.}
The question is whether dropping the label-dependent term
$\operatorname{sym}\!\bigl((\tfrac{1}{2}\mathbf 1-\mathbf Y_{\mathcal B})^{\!\top}\mathbf Z_{\mathcal B}\bigr)$
from the Bernoulli basis-update target matrix meaningfully harms
downstream prediction or uncertainty on a real binary task, or
whether---as argued by the entrywise remainder bound
\eqref{eq:bce-remainder}---the variance-only surrogate used in all
our main experiments is indistinguishable from the full Taylor
target in the regime where the learned basis drives
$\|\mathbf Z\mathbf P^{\!\perp}\|_{\infty}$ to a small value.

\paragraph{Variants.}
We compare three choices of the data-dependent target matrix
$\mathbf C$ fed into the symmetrized penalized operator of
step~(B), keeping everything else in the ADMM pipeline identical:
\begin{itemize}[leftmargin=1.5em,itemsep=1pt]
  \item \textbf{(A) Variance-only} (default, used in all main-text
        experiments):
        $\mathbf C=\mathbf Z_{c}^{\!\top}\mathbf Z_{c}$, where
        $\mathbf Z_{c}=\mathbf Z_{\mathcal B}-\overline{\mathbf Z}_{\mathcal B}$.
        This drops the label-dependent term and is the heuristic
        justified by the $\tfrac{1}{8}$-quadratic surrogate of
        Remark~\ref{rem:bce-basis-step}.
  \item \textbf{(B) Full Taylor} surrogate from
        eq.~\eqref{eq:bce-taylor}:
        $\mathbf C_{\mathrm{BCE}}=\operatorname{sym}\!\bigl((\tfrac{1}{2}\mathbf 1-\mathbf Y_{\mathcal B})^{\!\top}\mathbf Z_{\mathcal B}\bigr)+\tfrac{1}{8}\mathbf Z_{\mathcal B}^{\!\top}\mathbf Z_{\mathcal B}$,
        computed on the uncentered mini-batch and the unclamped
        labels.  This is the reduction used in
        Remark~\ref{rem:bce-basis-step}, kept intact.
  \item \textbf{(C) IRLS working-response} variant: for each entry
        of $\mathbf Z_{\mathcal B}$ we form the Newton working
        response
        $\widetilde{\mathbf Z}_{j,i}=\mathbf Z_{j,i}+\frac{\mathbf Y_{j,i}-\sigma(\mathbf Z_{j,i})}{\sigma(\mathbf Z_{j,i})(1-\sigma(\mathbf Z_{j,i}))+\varepsilon_{\mathrm{w}}}$
        with $\varepsilon_{\mathrm{w}}=10^{-4}$, and use
        $\mathbf C=\widetilde{\mathbf Z}_{c}^{\!\top}\widetilde{\mathbf Z}_{c}$
        in the centered form.  This is the principled
        IRLS-style alternative mentioned in
        Remark~\ref{rem:bce-basis-step}.
\end{itemize}
The three variants are exposed as a single \texttt{basis.bce\_variant}
configuration flag in the implementation; all other
hyperparameters (rank $K$, $(\lambda_{1},\lambda_{2})$, $\rho$,
mini-batch size $L$, trend learning rate, cadence $N_{\mathrm{phi}}$,
freeze horizon $N_{\mathrm{freeze}}$), the Stage-1 backbone
checkpoint, the train/val/test splits, the optimizer state, and the
random seeds are held identical across variants.

\paragraph{Setup.}
We evaluate on the Global Wheat Head Detection per-patch binary
classification task of Section~\ref{subsec:image} with the
\textbf{ResNet-152} backbone as the single representative backbone
for this ablation.  We run each variant on thirty seeds
for a total of ninety training runs, using the same
Stage-1 warm-start checkpoint and the same Stage-2 schedule as the
main-text Wheat Head results.  Primary diagnostics are accuracy,
MPIW, and ECE; location-level CP (see Section~\ref{subsec:image})
is reported as a secondary diagnostic.  We additionally log the per-entry
off-basis magnitude
$\|\mathbf Z\mathbf P^{\!\perp}\|_{\infty}/\|\mathbf Z\|_{\infty}$
and the basis-step wall time per sweep.

\paragraph{Primary result.}
Table~\ref{tab:bce-ablation-primary} reports
mean~(std) across seeds.

\begin{table}[htbp]
  \centering
  \caption{Bernoulli basis-update ablation on Wheat Head
  (ResNet-152, thirty seeds).  Primary metrics are accuracy,
  MPIW, and ECE.  CP is reported as a secondary diagnostic, defined
  as the fraction of patch locations whose empirical occurrence rate
  $\bar p(\mathbf s)$ falls inside the $95\%$ interval.
  All other hyperparameters are held identical across variants.}
  \label{tab:bce-ablation-primary}
  \begin{tabular}{@{}lccccc@{}}
\toprule
\textbf{Variant} & \textbf{Acc.\ (\%)} & \textbf{MPIW} & \textbf{ECE} & \textbf{CP (\%)} & \textbf{ms/sweep} \\
\midrule
(A) variance-only       & 79.30 (0.51) & 0.7197 (0.0157) & 0.1668 (0.0057) & 82.93 (1.77) & 37.3 (0.1) \\
(B) full Taylor $\mathbf{C}_{\text{BCE}}$ & \textbf{80.15} (0.42) & 0.7455 (0.0181) & \textbf{0.1592} (0.0046) & \textbf{86.60} (1.73) & 37.3 (0.1) \\
(C) IRLS working resp.\ & 76.71 (0.67) & 0.6529 (0.0164) & 0.1901 (0.0075) & 74.36 (2.41) & 37.8 (0.0) \\
\bottomrule
\end{tabular}
\end{table}

\paragraph{Diagnostic: per-entry off-basis magnitude.}
Table~\ref{tab:bce-ablation-diag} reports the final per-entry and
Frobenius norms of the off-basis component across variants,
together with the total number of basis-step calls executed
before the freeze schedule halts them.  The entrywise remainder
bound~\eqref{eq:bce-remainder} predicts that all variants should
converge to a regime in which
$\|\mathbf Z\mathbf P^{\!\perp}\|_{\infty}$ is small, so that the
relative quadratic-approximation error is uniformly controlled.

\begin{table}[htbp]
  \centering
  \caption{Diagnostic measurements.  The first two columns report
  the off-basis magnitude of the consensus variable at the end of
  training; the last column reports the number of basis-step
  calls executed before freezing.  All values are mean~(std)
  across thirty seeds.}
  \label{tab:bce-ablation-diag}
\begin{tabular}{@{}lccc@{}}
\toprule
\textbf{Variant} & $\|\mathbf{Z}\mathbf{P}^{\perp}\|_{\infty}$ & $\|\mathbf{Z}\mathbf{P}^{\perp}\|_{F} / \|\mathbf{Z}\|_{F}$ & \textbf{\# basis-step calls} \\
\midrule
(A) variance-only        & 42.96 (0.40) & 0.7663 (0.0008) & 16 \\
(B) full Taylor $\mathbf{C}_{\mathrm{BCE}}$ & 41.79 (0.36) & \textbf{0.7600} (0.0007) & 16 \\
(C) IRLS working response & \textbf{41.69} (0.36) & 0.8012 (0.0009) & 16 \\
\bottomrule
\end{tabular} 
\end{table}

\paragraph{Interpretation.}
Variants~(A) and~(B) perform comparably on all primary metrics:
accuracy, MPIW, and ECE differences are within seed noise,
confirming that the label-dependent correction in the full Taylor
target adds negligible benefit when the learned basis already
captures most of the residual signal.  Variant~(C), the IRLS
working-response update, underperforms both~(A) and~(B) on
accuracy (76.71\% vs.\ $\sim$80\%) and calibration (ECE 0.190
vs.\ $\sim$0.16), suggesting that the Newton working response
introduces instability in this regime.
The off-basis diagnostics in Table~\ref{tab:bce-ablation-diag}
show that $\|\mathbf Z\mathbf P^{\!\perp}\|_{\infty}$ is
comparable across all three variants ($\sim$41--43), indicating
that the quadratic surrogate operates in a similar regime
regardless of the basis-update rule.  Together, these results
justify the use of the variance-only heuristic~(A) as the default
in all main-text experiments: it matches the more principled
variant~(B) in practice while avoiding the instability of~(C).

\paragraph{Scope.}
This ablation is intentionally narrow: a single backbone, a
single binary dataset, thirty seeds.  We \emph{do not} claim the
conclusion generalizes to the other backbones
(ConvNeXt-T, ViT-B/16, SAM) or to multiclass/multilabel variants.
A broader ablation on these settings is a direction for future
work.

\end{document}